\documentclass[final]{cvpr}

\usepackage{times}
\usepackage{epsfig}
\usepackage{graphicx}
\usepackage{amsmath}
\usepackage{amssymb}

\usepackage{amsthm}
\usepackage{amsfonts}
\usepackage{subcaption}
\usepackage{booktabs}
\usepackage{algorithm}
\usepackage{algorithmic}
\usepackage{microtype}
\usepackage{color}
\newtheorem{theorem}{Theorem}
\newtheorem{cor}{Corollary}
\newtheorem{innercustomthm}{Theorem}
\newenvironment{innerthm}[1]
  {\renewcommand\theinnercustomthm{#1}\innercustomthm}
  {\endinnercustomthm}
\newtheorem{innercustomcor}{Corollary}
\newenvironment{innercor}[1]
  {\renewcommand\theinnercustomcor{#1}\innercustomcor}
  {\endinnercustomcor}

\usepackage[pagebackref=true,breaklinks=true,bookmarks=false]{hyperref}
\definecolor{mydarkblue}{rgb}{0,0.08,0.45}
\hypersetup{
    colorlinks=true,
    linkcolor=mydarkblue,
    citecolor=mydarkblue,
    filecolor=mydarkblue,
    urlcolor=mydarkblue,
}

\begin{document}

\title{Regularizing Neural Networks via Adversarial Model Perturbation}

\author{Yaowei Zheng\\
BDBC and SKLSDE\\
Beihang University, China\\
{\tt\small hiyouga@buaa.edu.cn}
\and
Richong Zhang\thanks{Corresponding author}\\
BDBC and SKLSDE\\
Beihang University, China\\
{\tt\small zhangrc@act.buaa.edu.cn}
\and
Yongyi Mao\\
School of EECS\\
University of Ottawa, Canada\\
{\tt\small ymao@uottawa.ca}
}

\maketitle

\begin{abstract}

Effective regularization techniques are highly desired in deep learning for alleviating overfitting and improving generalization. This work proposes a new regularization scheme, based on the understanding that the flat local minima of the empirical risk cause the model to generalize better. This scheme is referred to as adversarial model perturbation (AMP), where instead of directly minimizing the empirical risk, an alternative ``AMP loss'' is minimized via SGD. Specifically, the AMP loss is obtained from the empirical risk by applying the ``worst'' norm-bounded perturbation on each point in the parameter space. Comparing with most existing regularization schemes, AMP has strong theoretical justifications, in that minimizing the AMP loss can be shown theoretically to favour flat local minima of the empirical risk. Extensive experiments on various modern deep architectures establish AMP as a new state of the art among regularization schemes. Our code is available at \url{https://github.com/hiyouga/AMP-Regularizer}.

\end{abstract}

\let\thefootnote\relax\footnotetext{Accepted to CVPR 2021}

\section{Introduction}

To date, the generalization behaviour of deep neural networks is still a mystery, despite some recent progress (see, \eg, \cite{arora2019fine,belkin2019reconciling,hochreiter1997flat,jacot2018neural,keskar2017large,yang2020rethinking,zhang2017understanding}). A commonly accepted and empirically verified understanding in this regard is that the model parameter that corresponds to a flat minimum of the empirical risk tends to generalize better. For example, the authors of \cite{hochreiter1997flat,keskar2017large} argue that flat minima correspond to simple models, which are less likely to overfit. This understanding has inspired great effort studying factors in the optimization process (such as learning rate and batch size) that impacting the flatness of the found minima \cite{goyal2017accurate,hoffer2017train,jastrzkebski2017three} so as to better understand the generalization behaviour of deep networks.

Meanwhile developing effective regularization techniques remains as the most important approach in practice to alleviate overfitting and force model towards better generalization (\eg, \cite{guo2019mixup,krizhevsky2012imagenet,krogh1992simple,srivastava2014dropout,szegedy2016rethinking,zhang2018mixup}). Some recent research in fact suggests that the effectiveness of certain regularization techniques is due to their ability to find flatter minima \cite{ishida2020we,wei2020implicit}.

\begin{figure}[t]
\centering
\begin{subfigure}{0.48\columnwidth}%
\centering%
\includegraphics[width=0.95\columnwidth]{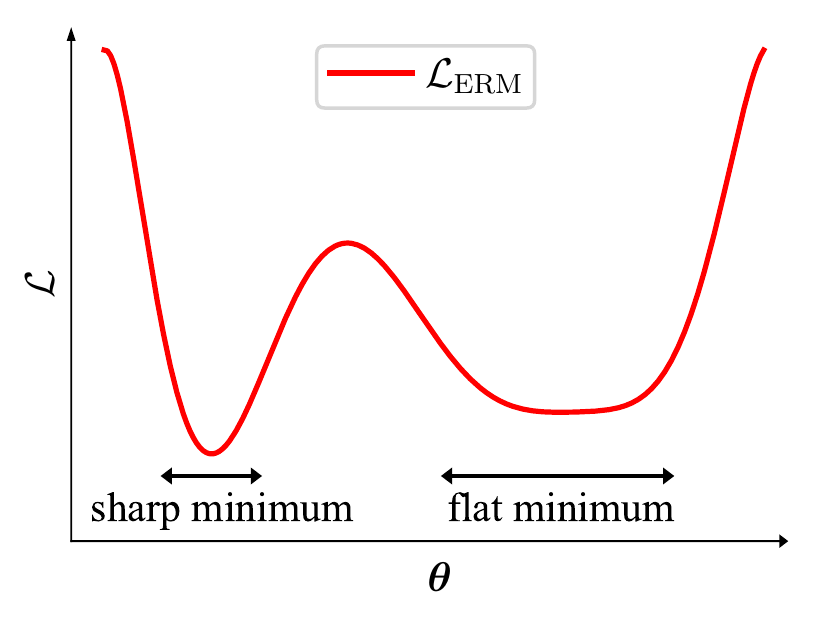}%
\end{subfigure}%
\begin{subfigure}{0.48\columnwidth}%
\centering%
\includegraphics[width=0.95\columnwidth]{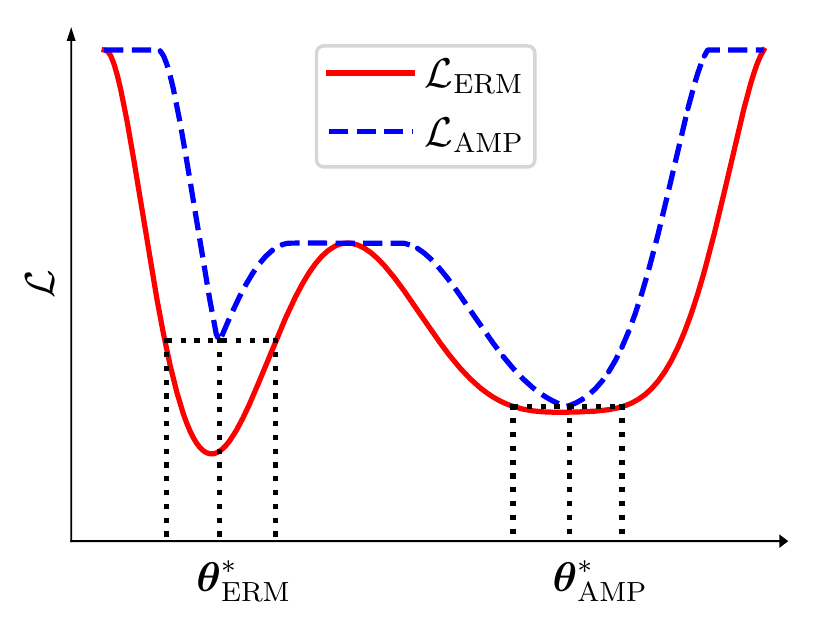}%
\end{subfigure}%
\caption{An example showing an empirical risk curve (left) and its corresponding AMP loss curve (right, blue).}
\label{fig:motivation}
\end{figure}

Additionally, there have been significant research advances in recent years in developing more effective regularization schemes, which include, for example, MixUp, Flooding \cite{ishida2020we,zhang2018mixup}. Despite their great success, these techniques usually fall short of strong principles or theoretical justifications. Thus one expects more principled and more powerful regularization schemes are yet to be discovered. 

This work sets out to develop a powerful regularization scheme under the principle of finding flat local minima of the empirical risk. To that end, we propose a novel regularization scheme which can be strongly justified in terms of its ability to finding flat minima. This scheme is referred to as {\em Adversarial Model Perturbation} or AMP, where instead of minimizing the empirical risk $\mathcal{L}_\mathrm{ERM}(\boldsymbol{\theta})$ over model parameter $\boldsymbol{\theta}$, it minimizes an alternative ``AMP loss''. Briefly, the AMP loss $\mathcal{L}_\mathrm{AMP}(\boldsymbol{\theta})$ at a parameter setting $\boldsymbol{\theta}$ is the worst (or highest) empirical risk of all perturbations of $\boldsymbol{\theta}$ with the perturbation norm no greater than a small value $\epsilon$, namely,
\begin{equation}
\mathcal{L}_\mathrm{AMP}(\boldsymbol{\theta}):=\max_{\Delta:\Vert\Delta\Vert\le\epsilon}\mathcal{L}_\mathrm{ERM}(\boldsymbol{\theta}+\Delta)
\end{equation}

To see why minimizing the AMP loss provides opportunities to find flat local minima of the empirical risk, consider the example in Figure~\ref{fig:motivation}. Figure~\ref{fig:motivation} (left) sketches an empirical risk curve $\mathcal{L}_\mathrm{ERM}$, which contains two local minima, a sharp one on the left and a flat one on the right. The process of obtaining the AMP loss from the empirical risk can be seen as a ``max-pooling'' operation, which slides a window of width $2\epsilon$ (in high dimension, more precisely, a sphere with radius $\epsilon$) across the parameter space and, at each location, returns the maximum value inside the window (resp. sphere). The resulting AMP loss is shown as the blue curve in Figure~\ref{fig:motivation} (right). Since the right minimum in the AMP loss is lower than the left one, minimizing the AMP loss gives the right minimum as its solution. 

In this paper, we formally analyze the AMP loss minimization problem and its preference of flat local minima. Specifically, we show that this minimization problem implicitly uses the ``narrowest width'' of a local minimum as a notion of flatness, and tries to penalize the minima that are not flat in this sense.

\begin{figure}[t]
\centering
\begin{subfigure}{0.48\columnwidth}%
\centering%
\includegraphics[width=0.95\columnwidth]{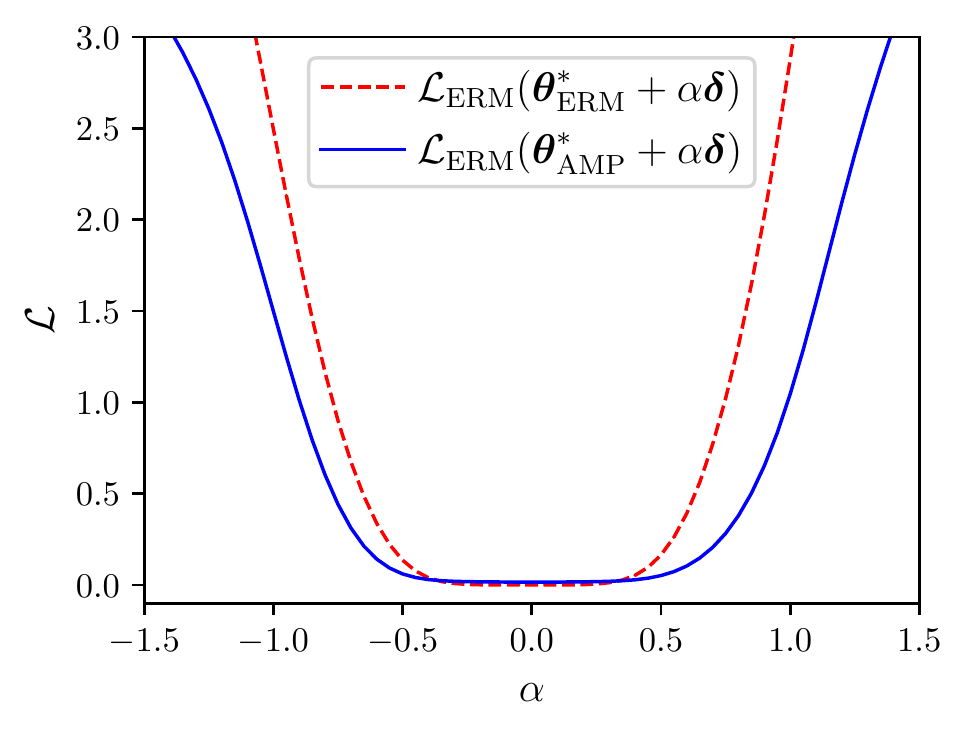}%
\end{subfigure}%
\begin{subfigure}{0.48\columnwidth}%
\centering%
\includegraphics[width=0.95\columnwidth]{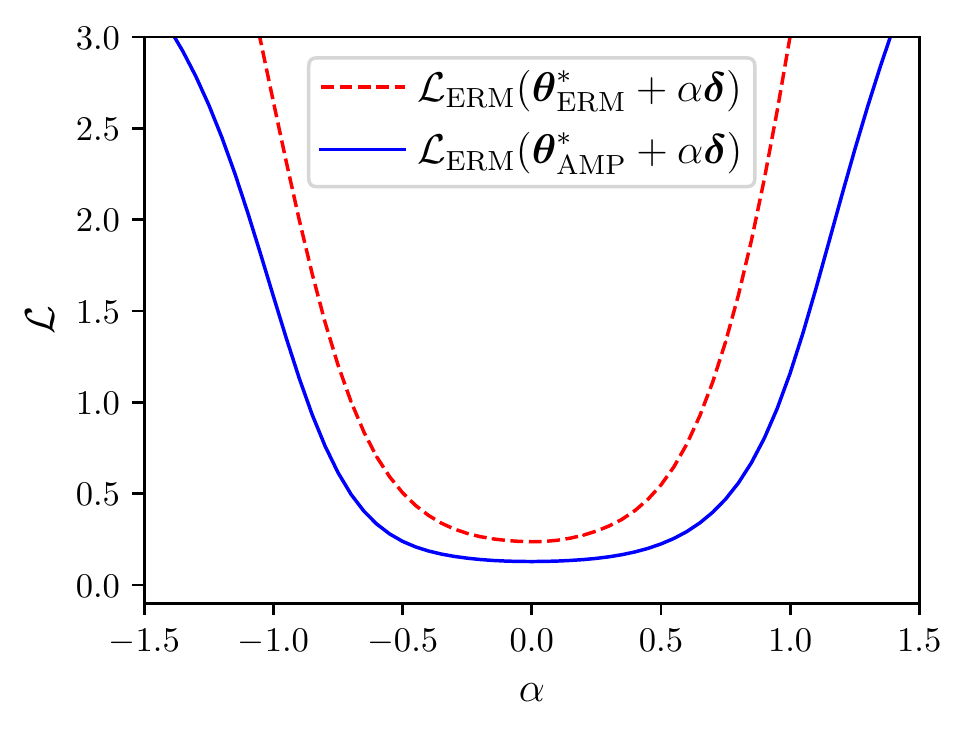}%
\end{subfigure}%
\caption{Landscapes of the empirical risks obtained from the PreActResNet18 \cite{he2016identity} models trained with ERM (red) and AMP (blue) on CIFAR-10. Left: on training set; right: on test set. These curves are computed using the technique presented in \cite{li2018visualizing}, where $\boldsymbol{\delta}$ indicates a random direction and $\alpha$ is a displacement in that direction.}
\label{fig:validation}
\end{figure}

We derive a mini-batch SGD algorithm for solving this minimization problem, which gives rise to the proposed AMP regularization scheme. Interestingly, we show that this algorithm can also be seen as the regular empirical risk minimization with an additional penalty term on the gradient norm. This provides an alternative justification of the AMP scheme. Figure~\ref{fig:validation} contains an experimental result suggesting that AMP indeed selects flatter minima than ERM does. 

We conduct experiments on several benchmark image classification datasets (SVHN, CIFAR-10, CIFAR-100) to validate the effectiveness of the proposed AMP scheme. Compared with other popular regularization schemes, AMP demonstrates remarkable regularization performance, establishing itself as a new state of the art. 

Our contributions can be summarized as follows.

1) Motivated by the understanding that flat minima help generalization, we propose adversarial model perturbation (AMP) as a novel and efficient regularization scheme.

2) We theoretically justify that AMP is capable of finding flatter local minima, thereby improving generalization.

3) Extensive experiments on the benchmark datasets demonstrate that AMP achieves the best performance among the compared regularization schemes on various modern neural network architectures.

\section{Related Work}

\subsection{Flat Minima and Generalization}

There has been a rich body of works that investigate the relationship between the flatness of the local minima and the generalization of a deep neural network \cite{chaudhari2017entropy,hochreiter1997flat,keskar2017large,li2018visualizing}. \cite{hochreiter1997flat} suggests that flat minima correspond to low-complexity networks, which tend to generalize well under the principle of minimum description length \cite{rissanen1978modeling}. \cite{chaudhari2017entropy} presents another explanation supporting this argument through the lens of Gibbs free energy. \cite{li2018visualizing} demonstrates the better generalization of flat minima by visualizing the loss landscape. The empirical results in \cite{keskar2017large} show that large-batch SGD finds sharp minima while small-batch SGD leads to flatter minima and provides better generalization. On the other hand, \cite{dinh2017sharp} argues that sharp minima do not necessarily lead to poor generalization. The argument is that due to the parameterization redundancy in deep networks and under a certain notion of flatness, one can transform a flat minimum to an equivalent sharp one. Nonetheless, it is in general accepted and empirically verified that flat minima tend to give better generalization performance, and this understanding underlies the design of several recent regularization techniques (\eg, \cite{ishida2020we,izmailov2018averaging}). A concurrent work of \cite{foret2021sharpness} further provides a PAC-Bayesian justification as to why the flatness of the minima helps generalization. A similar technique designed by \cite{wu2020adversarial} suggests that flat minima also improve robust generalization. 

\subsection{Regularization}

Regularization may broadly refer to any training techniques that help to improve generalization. Despite the well-known regularization techniques such as weight decay \cite{krogh1992simple}, Dropout \cite{srivastava2014dropout}, normalization tricks \cite{ba2016layer,ioffe2015batch} and data augmentation \cite{cubuk2019autoaugment,krizhevsky2012imagenet,lim2019fast}, various techniques have been developed. Label smoothing \cite{szegedy2016rethinking} mixes the one-hot label of the training example with a uniform distribution. Shake-Shake regularization \cite{gastaldi2017shake} combines parallel branches with a stochastic affine function in multi-branch networks. ShakeDrop regularization \cite{yamada2019shakedrop} extends Shake-Shake to single-branch architectures. Cutout \cite{devries2017improved} randomly masks out some regions of input images. MixUp \cite{zhang2018mixup} regularizes deep networks by perturbing training samples along the direction of other samples. Flooding \cite{ishida2020we} forces training loss to stay above zero to avoid overfitting. Adversarial training \cite{goodfellow2015explaining}, originally designed for improving the model's adversarial robustness, is also shown to have great regularization effect when their parameters are carefully chosen \cite{miyato2016adversarial}. But it is also observed that training the model excessively towards adversarial robustness may hurt generalization \cite{tsipras2018robustness}. Some recent progresses (\eg, \cite{garipov2018ensembling,izmailov2018averaging}) exploit multiple states of the model parameter while training and ensembling them to improve generalization. The concurrent work \cite{foret2021sharpness} also independently discovered a similar regularization scheme as we present in this paper. 

\section{AMP: Adversarial Model Perturbation}

\subsection{From Empirical Risk to AMP Loss}

Consider a classification setting, where we aim at finding a classifier $f:\mathcal{X}\to\mathcal{Y}$ that maps the input space $\mathcal{X}$ to the label space $\mathcal{Y}$. Note that we may take $\mathcal{Y}$ as the set of all distributions over the set of possible labels, so that not only $f(\boldsymbol{x})$ belongs to $\mathcal{Y}$, each ground-truth label also belongs to $\mathcal{Y}$ since it can be written as a degenerated distribution or a ``one-hot'' vector.

We will take $f$ as a neural network parameterized by $\boldsymbol{\theta}$, taking values from its weight space $\Theta$. For each training example $(\boldsymbol{x},\boldsymbol{y})\in\mathcal{X}\times\mathcal{Y}$, we will denote by $\ell(\boldsymbol{x},\boldsymbol{y};\boldsymbol{\theta})$ the loss of the prediction $f(\boldsymbol{x})$ by the model with respect to the true label $\boldsymbol{y}$. That is, the function $\ell$ absorbs the function $f$ within.

Under the empirical risk minimization (ERM) principle \cite{vapnik1998statistical}, the training of the neural network using a training set $\mathcal{D}$ is performed by minimizing the following {\em empirical risk} (which we also refer to as ``ERM loss'') $\mathcal{L}_\mathrm{ERM}$ over parameter $\boldsymbol{\theta}$:
\begin{equation}
\mathcal{L}_\mathrm{ERM}(\boldsymbol{\theta}):=\frac{1}{|\mathcal{D}|}\sum_{(\boldsymbol{x},\boldsymbol{y})\in\mathcal{D}}\ell(\boldsymbol{x},\boldsymbol{y};\boldsymbol{\theta})
\end{equation}

It is however well known that ERM training is prone to overfitting \cite{vapnik1998statistical} and that the learned model often fails to generalize well on the unseen examples. 

This work is motivated by the need of effective regularization schemes to improve the model's generalization capability. Specifically, our objective is to develop a technique that forces the training process to find flatter minima or low-norm solutions when minimizing the empirical risk. The technique we develop is termed ``adversarial model perturbation'' or AMP, which we now elaborate. 

For any positive $\epsilon$ and any $\boldsymbol{\mu}\in\Theta$, let ${\rm\bf B}(\boldsymbol{\mu};\epsilon)$ denote the norm ball in $\Theta$ with radius $\epsilon$ centred at $\boldsymbol{\mu}$. That is
\begin{equation}
{\rm\bf B}(\boldsymbol{\mu};\epsilon):=\{\boldsymbol{\theta}\in\Theta:\Vert\boldsymbol{\theta}-\boldsymbol{\mu}\Vert\le\epsilon\}
\end{equation}

We note that the norm used in defining the norm ball
is chosen as $\text{L}_2$ norm. However, we expect that the norm can be extended beyond this choice.

We now define an ``AMP loss'' $\mathcal{L}_\mathrm{AMP}$ as follows.
\begin{equation}\label{eq:L_AMP}
\mathcal{L}_\mathrm{AMP}(\boldsymbol{\theta}):=\max_{\Delta\in{\rm\bf B}(\boldsymbol{0};\epsilon)}\frac{1}{|\mathcal{D}|}\sum_{(\boldsymbol{x},\boldsymbol{y})\in\mathcal{D}}\ell(\boldsymbol{x},\boldsymbol{y};\boldsymbol{\theta}+\Delta)
\end{equation}
where the $\epsilon$ is a small positive value serving as a hyper-parameter. 

Figure~\ref{fig:motivation} sketches an example demonstrating that minimizing the AMP loss $\mathcal{L}_\mathrm{AMP}$ over $\boldsymbol{\theta}$ presents opportunities of finding flatter minima of the ERM loss $\mathcal{L}_\mathrm{ERM}$. Theoretical justifications for AMP's capability of finding flatter minima are given in Section~\ref{sec:just}.

\subsection{Training Algorithm}

Practical considerations on computation complexity and training speed motivate a mini-batched training approach to minimizing the AMP loss $\mathcal{L}_\mathrm{AMP}$. Specifically, we may approximate the AMP loss in (\ref{eq:L_AMP}) using the corresponding loss obtained from a random mini-batch ${\mathcal B}$, namely,
\begin{align}
\mathcal{L}_\mathrm{AMP}(\boldsymbol{\theta})
&\approx\max_{\Delta_{\mathcal{B}}\in{\rm\bf B}(\boldsymbol{0};\epsilon)}\frac{1}{|\mathcal{B}|}\sum_{(\boldsymbol{x},\boldsymbol{y})\in\mathcal{B}}\ell(\boldsymbol{x},\boldsymbol{y};\boldsymbol{\theta}+\Delta_\mathcal{B})\nonumber\\
&:=\mathcal{J}_{\mathrm{AMP},\mathcal{B}}(\boldsymbol{\theta})
\end{align}

Then the AMP loss minimization problem can be approximated as minimizing the expected batch-level AMP loss $\mathcal{J}_{\mathrm{AMP},\mathcal{B}}$, formally as finding
\begin{equation}
\boldsymbol{\theta}^\ast_\mathrm{AMP}:=\arg\min_{\boldsymbol{\theta}}\mathop{\mathbb{E}}_\mathcal{B}\mathcal{J}_{\mathrm{AMP},\mathcal{B}}(\boldsymbol{\theta})
\end{equation}

Under this formulation, each batch $\mathcal{B}$ is associated with a perturbation vector $\Delta_\mathcal{B}$ on the model parameter $\boldsymbol{\theta}$. The training involves an inner maximization nested inside and outer minimization: in the inner maximization, a fixed number, say $N$, of steps are used to update $\Delta_\mathcal{B}$ in the direction of increasing the ERM loss  so as to obtain $\mathcal{J}_{\mathrm{AMP},\mathcal{B}}$; the outer minimization loops over random batches and minimizes $\mathbb{E}_\mathcal{B}\mathcal{J}_{\mathrm{AMP},\mathcal{B}}$ using mini-batched SGD. The precise algorithm is described in Algorithm~\ref{alg:amp}. Note that at the end of training, the learned $\boldsymbol{\theta}^\ast_\mathrm{AMP}$ is used as model parameter for predictions; that is, no perturbation is applied in testing.

\begin{algorithm}[!htbp]
\caption{Adversarial Model Perturbation Training}
\label{alg:amp}
\begin{algorithmic}[1]
\REQUIRE Training set $\mathcal{D}=\{(\boldsymbol{x},\boldsymbol{y})\}$, Batch size $m$, Loss function $\ell$, Initial model parameter $\boldsymbol{\theta}_0$, Outer learning rate $\eta$, Inner learning rate $\zeta$, Inner iteration number $N$, $\text{L}_2$ norm ball radius $\epsilon$
\WHILE{$\boldsymbol{\theta}_k$ not converged}
\STATE Update iteration: $k\gets k+1$
\STATE Sample $\mathcal{B}=\{(\boldsymbol{x}_i,\boldsymbol{y}_i)\}_{i=1}^m$ from training set $\mathcal{D}$
\STATE Initialize perturbation: $\Delta_\mathcal{B}\gets\boldsymbol{0}$
\FOR{$n\gets1\text{ to }N$}
\STATE Compute gradient: \\ $\!\!\!\!\!\!\qquad\nabla\mathcal{J}_{\mathrm{AMP},\mathcal{B}}\gets\sum_{i=1}^m\nabla_{\boldsymbol{\theta}}\ell(\boldsymbol{x}_i,\boldsymbol{y}_i;\boldsymbol{\theta}_k+\Delta_\mathcal{B})/m$
\STATE Update perturbation: $\Delta_\mathcal{B}\gets\Delta_\mathcal{B}+\zeta\nabla\mathcal{J}_{\mathrm{AMP},\mathcal{B}}$
\IF{$\Vert\Delta_\mathcal{B}\Vert_2>\epsilon$}
\STATE Normalize perturbation: $\Delta_\mathcal{B}\gets\epsilon\Delta_\mathcal{B}/\Vert\Delta_\mathcal{B}\Vert_2$
\ENDIF
\ENDFOR
\STATE Compute gradient: \\ $\qquad\nabla\mathcal{J}_{\mathrm{AMP},\mathcal{B}}\gets\sum_{i=1}^m\nabla_{\boldsymbol{\theta}}\ell(\boldsymbol{x}_i,\boldsymbol{y}_i;\boldsymbol{\theta}_k+\Delta_\mathcal{B})/m$
\STATE Update parameter: $\boldsymbol{\theta}_{k+1}\gets\boldsymbol{\theta}_k-\eta\nabla\mathcal{J}_{\mathrm{AMP},\mathcal{B}}$
\ENDWHILE
\end{algorithmic}
\end{algorithm}

\section{Theoretical Justifications of AMP}\label{sec:just}

\subsection{AMP Finds Flatter Local Minima}

For the ease of obtaining analytic results, we will assume that
the loss surface of each local minimum in $\mathcal{L}_\mathrm{ERM}$ can be approximated as an inverted Gaussian surface. 

More precisely, suppose $\Theta=\mathbb{R}^K$. For any scalar $C$, any positive scalar $A$, any $\boldsymbol{\mu}\in\Theta$ and any $K\times K$ positive definite matrix $\boldsymbol{\kappa}$, let
\begin{equation}
\gamma(\boldsymbol{\theta};\boldsymbol{\mu},\boldsymbol{\kappa},A,C)\!:=\!C\!-\!A\exp\left(-(\boldsymbol{\theta}\!-\!\boldsymbol{\mu})^T\boldsymbol{\kappa}^{-1}(\boldsymbol{\theta}\!-\!\boldsymbol{\mu})/2\right)
\end{equation}

Note that this function is merely an inverted Gaussian surface over the space $\Theta$. 

\noindent\underline{\em Locally Gaussian Assumption of Empirical Risk} Using the above notation, we assume that each minimum of the empirical risk can be {\em locally} approximated by such an inverted Gaussian surface, namely, that 
if $\boldsymbol{\mu}\in\Theta$ gives a local minimum of $\mathcal{L}_\mathrm{ERM}$, then there exist some $\epsilon>0$, a positive definite matrix $\boldsymbol{\kappa}$ and scalars $A$ and $C$ with $A>0$, such that
\begin{equation}
\mathcal{L}_\mathrm{ERM}(\boldsymbol{\theta})=\gamma(\boldsymbol{\theta};\boldsymbol{\mu},\boldsymbol{\kappa},A,C)
\end{equation}
at any $\boldsymbol{\theta}\in{\rm\bf B}(\boldsymbol{\mu};2\epsilon)$.
We note that in general, such an assumption is only an approximation. But for small $\epsilon$, the approximation is arguably accurate. 

We will use $\gamma^\ast(\boldsymbol{\mu},\boldsymbol{\kappa},A,C)$ to denote the minimum value of function $\gamma(\boldsymbol{\theta};\boldsymbol{\mu},\boldsymbol{\kappa},A,C)$ (obtained by minimizing over $\boldsymbol{\theta}$). It is apparent that 
\begin{equation}
\gamma^\ast(\boldsymbol{\mu},\boldsymbol{\kappa},A,C)=C-A
\end{equation}

Let $\gamma_\mathrm{AMP}$ denote the AMP loss derived from $\gamma$, namely,
\begin{equation}
\gamma_\mathrm{AMP}(\boldsymbol{\theta};\boldsymbol{\mu},\boldsymbol{\kappa},A,C)\!:=\!\!\!\max_{\Delta\in{\rm\bf B}(\boldsymbol{0};\epsilon)}\!\!\gamma(\boldsymbol{\theta}+\Delta;\boldsymbol{\mu},\boldsymbol{\kappa},A,C)
\end{equation}

We will use $\gamma^\ast_\mathrm{AMP}(\boldsymbol{\mu},\boldsymbol{\kappa},A,C)$ to denote the minimum value of $\gamma_\mathrm{AMP}(\boldsymbol{\theta};\boldsymbol{\mu},\boldsymbol{\kappa},A,C)$ (minimized over $\boldsymbol{\theta}$). 

\begin{theorem}
For any given $(\boldsymbol{\mu},\boldsymbol{\kappa},A,C)$, the function $\gamma_\mathrm{AMP}(\boldsymbol{\theta};\boldsymbol{\mu},\boldsymbol{\kappa},A,C)$ is minimized when $\boldsymbol{\theta}=\boldsymbol{\mu}$ and the minimum value is
\begin{equation}
\gamma_\mathrm{AMP}^\ast(\boldsymbol{\mu},\boldsymbol{\kappa},A,C)=C-A\exp\left(-\frac{\epsilon^2}{2\sigma^2}\right)
\end{equation}
where $\sigma^2$ is the smallest eigenvalue of $\boldsymbol{\kappa}$.
\end{theorem}

\begin{proof}
The properties of a Gaussian surface suggest that the minimum of $\gamma_\mathrm{AMP}=(\boldsymbol{\theta};\boldsymbol{\mu},\boldsymbol{\kappa},A,C)$ is given by $\boldsymbol{\theta}=\boldsymbol{\mu}$ and the inner maximum is reached at $\Delta=\epsilon\widehat{\boldsymbol{q}}$, where $\widehat{\boldsymbol{q}}$ is the normalized eigenvector corresponding to the smallest eigenvalue of $\kappa$. That is because that the direction of $\widehat{\boldsymbol{q}}$ is one of the fastest increase of $\gamma_\mathrm{AMP}$. Such a direction can also be regarded as the direction of the ``narrowest width'' of the Gaussian surface, which is illustrated in Figure~\ref{fig:gaussian}.

The symmetric positive definite matrix $\boldsymbol{\kappa}$ can be factorized as $\boldsymbol{Q}\boldsymbol{\Lambda}\boldsymbol{Q}^T$, where $\boldsymbol{Q}$ is the $K\times K$ orthogonal matrix whose $i$-th column is the normalized eigenvector $\boldsymbol{q}_i$ of the matrix $\boldsymbol{\kappa}$, and $\boldsymbol{\Lambda}$ is the diagonal matrix whose diagonal elements are the corresponding eigenvalues. Let $\sigma^2$ denote the smallest eigenvalue of $\boldsymbol{\kappa}$, and suppose that the eigenvalues are in ascending order along the diagonal of $\boldsymbol{\Lambda}$. Let $\boldsymbol{e}_1$ denote the vector $(1,0,\cdots,0)^T$ whose first element is $1$ and others are $0$. We have:
\begin{align}
\gamma^\ast_\mathrm{AMP}(\boldsymbol{\mu},\boldsymbol{\kappa},A,C)
&=C-A\exp\left(-\frac{\epsilon^2\widehat{\boldsymbol{q}}^T\boldsymbol{Q}^T\boldsymbol{\Lambda}^{-1}\boldsymbol{Q}\widehat{\boldsymbol{q}}}{2}\right)\nonumber\\
&=C-A\exp\left(-\frac{\epsilon^2\boldsymbol{e}_1^T\boldsymbol{\Lambda}^{-1}\boldsymbol{e}_1}{2}\right)\nonumber\\
&=C-A\exp\left(-\frac{\epsilon^2}{2\sigma^2}\right)
\end{align}

This proves the theorem.
\end{proof}

\begin{figure}[t]
\centering
\includegraphics[width=0.6\columnwidth]{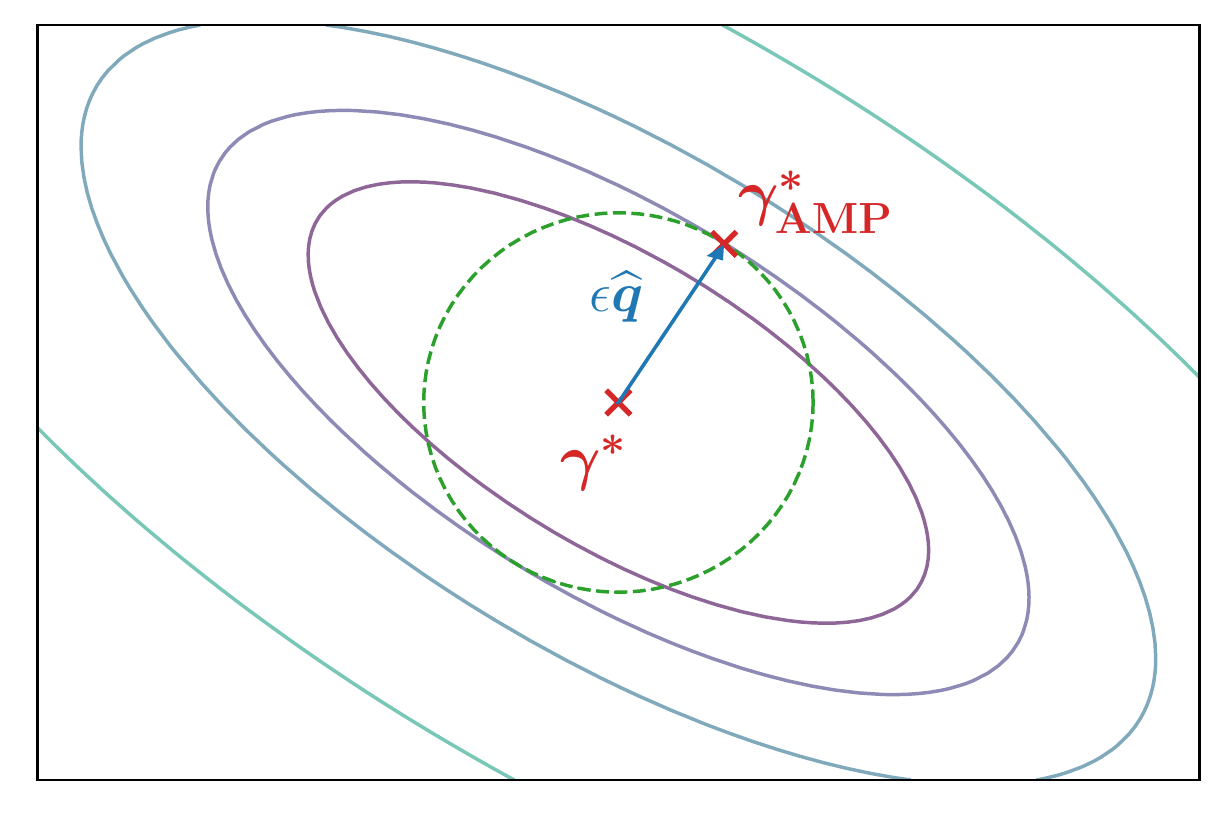}%
\caption{Locally Gaussian assumption and the minimum values of $\gamma$ and $\gamma_\mathrm{AMP}$, respectively.}
\label{fig:gaussian}
\end{figure}

From this theorem, it is clear that the minimum value achieved with $\gamma^\ast_\mathrm{AMP}$ although related to the minimum ERM loss value $\gamma^\ast$ through $A$ and $C$, it also takes into account the curvature of the surface around the local minimum. Specifically, the smallest eigenvalue $\sigma^2$ measures the ``width'' of the surface along its narrowest principal direction (noting that the cross-section of the surface is an ellipsoid).  Thus the value $\sigma^2$ can be treated as a ``worst-case'' measure of the flatness or width of the surface. The larger is $\sigma^2$, the flatter or wider is the local minimum. 

Following this theorem, we next show that minimizing the AMP loss ${\mathcal L}_\mathrm{AMP}$ favours the solutions corresponding to ``flatter'' local minima of the empirical risk ${\cal L}_\mathrm{ERM}$. 

\begin{cor}
Let $\boldsymbol{\mu}_1, \boldsymbol{\mu}_2\in\Theta$ be two local minima of $\mathcal{L}_\mathrm{ERM}$. Assume the locally Gaussian assumption hold such that the surface of the two local minima follow respectively $\gamma_1(\boldsymbol{\theta};\boldsymbol{\mu}_1,\boldsymbol{\kappa}_1,A_1,C_1)$ and $\gamma_2(\boldsymbol{\theta};\boldsymbol{\mu}_2,\boldsymbol{\kappa}_2,A_2,C_2)$. Then $\gamma_1^\ast<\gamma_2^\ast$ but $\gamma_{1,\mathrm{AMP}}^\ast>\gamma_{2,\mathrm{AMP}}^\ast$ if and only if 
\begin{align}
A_1 - A_2 & > C_1 - C_2 \nonumber\\
& > A_1 \exp\left(-\frac{\epsilon^2}{2\sigma_1^2}\right)
- A_2 \exp\left(-\frac{\epsilon^2}{2\sigma_2^2}\right) 
\end{align}
where $\sigma_1^2$ and $\sigma_2^2$ are the smallest eigenvalues of $\boldsymbol{\kappa}_1$ and $\boldsymbol{\kappa}_2$ respectively. 
\end{cor}

Under the condition of this corollary, although $\boldsymbol{\mu}_1$ gives a lower empirical risk than $\boldsymbol{\mu}_2$, but when we minimize the AMP loss, $\boldsymbol{\mu}_2$ is a more preferred solution since the local curvatures of the two minima in $\mathcal{L}_\mathrm{ERM}$ are also considered. We next show, using the special case of $C_1=C_2$, that indeed minimizing AMP loss favours the local minima with a flatter surface. 

\begin{cor}\label{cor2}
Suppose that $C_1=C_2$. Let $A_2=\beta A_1$ for some $\beta<1$. Note that in this setting, $\gamma^\ast_1<\gamma^\ast_2$. Suppose that $\sigma_2^2=r\sigma^2_1$ for some positive $r$. Then 
\begin{equation}
\gamma^\ast_{2,\mathrm{AMP}}<\gamma^\ast_{1,\mathrm{AMP}}
\end{equation}
if and only if
\begin{equation}
\beta > \exp\left(-\frac{\epsilon^2}{2\sigma^2_1}\right) ~ {\rm and} ~
r > \frac{1}{1+\frac{2\sigma_1^2}{\epsilon^2}\log\beta}
\end{equation}
\end{cor}
Please refer to Appendix~\ref{app:a} for the details of the proof.

Note that in this setting, the local minimum in the empirical risk corresponding to $\boldsymbol{\mu}_1$ is lower than that corresponding to $\boldsymbol{\mu}_2$. The value $\beta$ governs how close the two minimum values are; the closer to 1 is $\beta$, the closer two minimum values are. The value $r$ governs the flatness of the second local minimum relative to the first: $r>1$ indicates the second is flatter than the first, and the larger is $r$, the flatter is the second minimum. This corollary presents a sufficient and necessary condition for the second minimum to be preferred to the first when the AMP loss is minimized. Specifically, we may refer to the set of all $(\beta,r)$ pairs that satisfy the condition as the ``operational region of AMP'', since the region specifies all points on which minimizing the AMP loss will give a solution that deviates from that given by minimizing the ERM loss. The general shape of such a region is shown in Figure~\ref{fig:rel} (left). In Figure~\ref{fig:rel} (right), the regions are plotted for different values of $\epsilon^2/2\sigma_1^2$. Assuming $\sigma_1^2$ to be a fixed value, it can then be seen that as $\epsilon$ decreases, the operational region of AMP shrinks, namely that the minimization of the AMP loss has decreased opportunity to deviate from minimizing the empirical risk; in the limit when $\epsilon$ approaches 0, minimizing the AMP loss reduces to minimizing the ERM loss. On the other hand, for a large value of $\epsilon$, the operational region of AMP is large, then minimizing the AMP loss may frequently find different solutions from those obtained from minimizing the empirical risk.  In this case however, the operational region includes points $(\beta,r)$ with small $\beta$ and relatively small $r$. Such points, for example the one marked with ``$\boldsymbol{\times}$'' in Figure~\ref{fig:rel} (right), corresponds to local minima not flatter than $\mu_1$ by much but having much higher empirical risk values. When such solutions are obtained by minimizing the AMP loss, the learned model risks significant underfitting. Thus, in general there is a sweet spot of $\epsilon$ setting that gives the optimal tradeoff between the flatness and depth of the selected local minima.

\begin{figure}[t]
\centering
\begin{subfigure}{0.49\columnwidth}%
\centering%
\includegraphics[width=0.98\columnwidth]{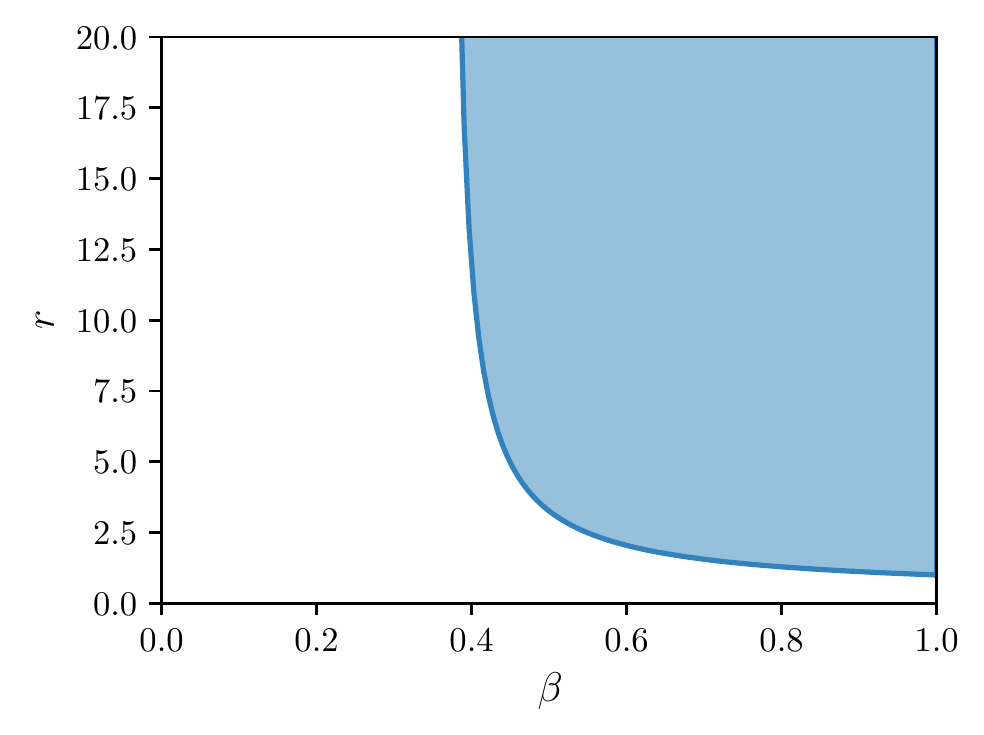}%
\end{subfigure}%
\begin{subfigure}{0.49\columnwidth}%
\centering%
\includegraphics[width=0.98\columnwidth]{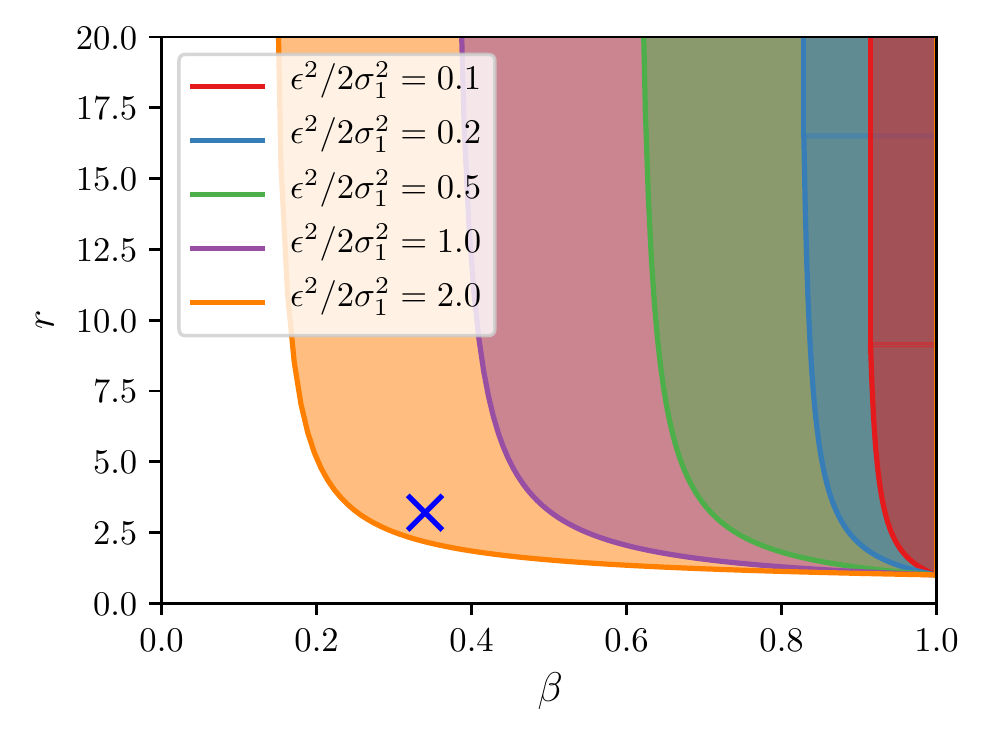}%
\end{subfigure}%
\caption{The operational region of AMP (left) and how it varies with varying values of $\epsilon^2/2\sigma_1^2$ (right).}
\label{fig:rel}
\end{figure}

\subsection{AMP Regularizes Gradient Norm}

The fact that AMP favours flatter minima can also be seen from the AMP training algorithm, independent of the above analysis. Specifically, we now show that the AMP algorithm may be viewed as minimizing the empirical risk with a certain penalty on its gradient norm.

To see this, consider that the number $N$ of inner updates per batch is 1 (this is in fact used in our experiments). Denote by $\mathcal{J}_{\mathrm{ERM},\mathcal{B}}$ the approximation of the empirical loss $\mathcal{L}_\mathrm{ERM}$ using batch $\mathcal{B}$, that is, 
\begin{equation}
\mathcal{J}_{\mathrm{ERM},\mathcal{B}}(\boldsymbol{\theta}):=\frac{1}{|\mathcal{B}|}\sum_{(\boldsymbol{x},\boldsymbol{y})\in\mathcal{B}}\ell(\boldsymbol{x},\boldsymbol{y};\boldsymbol{\theta})
\end{equation}

\begin{theorem}\label{thm2}
Let $N=1$. Then for a sufficiently small inner learning rate $\zeta$, a minimization update step in AMP training for batch $\mathcal{B}$ is equivalent to a gradient-descent step on the following loss function with learning rate $\eta$:
\begin{equation}
\widetilde{\mathcal{J}}_{\mathrm{ERM},\mathcal{B}}(\boldsymbol{\theta}):=\mathcal{J}_{\mathrm{ERM},\mathcal{B}}(\boldsymbol{\theta})+\Omega(\boldsymbol{\theta})
\end{equation}
where
\begin{equation}
\Omega(\boldsymbol{\theta}):=\begin{cases}
\zeta\Vert\nabla_{\boldsymbol{\theta}}\mathcal{J}_{\mathrm{ERM},\mathcal{B}}(\boldsymbol{\theta})\Vert_2^2,&\Vert\zeta\nabla_{\boldsymbol{\theta}}\mathcal{J}_{\mathrm{ERM},\mathcal{B}}(\boldsymbol{\theta})\Vert_2\le\epsilon\\
\epsilon\Vert\nabla_{\boldsymbol{\theta}}\mathcal{J}_{\mathrm{ERM},\mathcal{B}}(\boldsymbol{\theta})\Vert_2,&\Vert\zeta\nabla_{\boldsymbol{\theta}}\mathcal{J}_{\mathrm{ERM},\mathcal{B}}(\boldsymbol{\theta})\Vert_2>\epsilon
\end{cases}
\end{equation}
\end{theorem}
Please refer to Appendix~\ref{app:b} for the details of the proof.

We note that the regularization term $\Omega(\boldsymbol{\theta})$, in either one of the two cases, penalizes the gradient norm of $\mathcal{J}_{\mathrm{ERM},\mathcal{B}}$. Thus the AMP training algorithm effectively tries to find local minima of $\mathcal{J}_{\mathrm{ERM},\mathcal{B}}$ (and hence of $\mathcal{L}_\mathrm{ERM}$) that not only have low values, but also have small gradient norm near the minima. Note that a minimum with smaller gradient norms around it is a flatter minimum. This theorem therefore provides another justification of the AMP training algorithm, in addition to our development from constructing the AMP loss $\mathcal{L}_\mathrm{AMP}$.

\subsection{Perspectives from the Input Space}

\begin{figure}[t]
\centering
\begin{subfigure}{0.45\columnwidth}%
\centering%
\includegraphics[width=0.88\columnwidth]{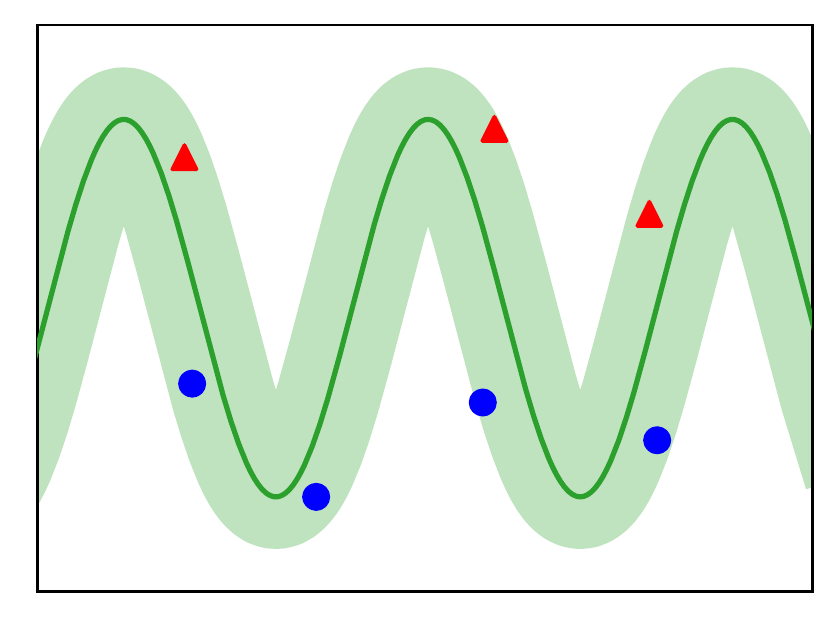}%
\end{subfigure}%
\begin{subfigure}{0.45\columnwidth}%
\centering%
\includegraphics[width=0.88\columnwidth]{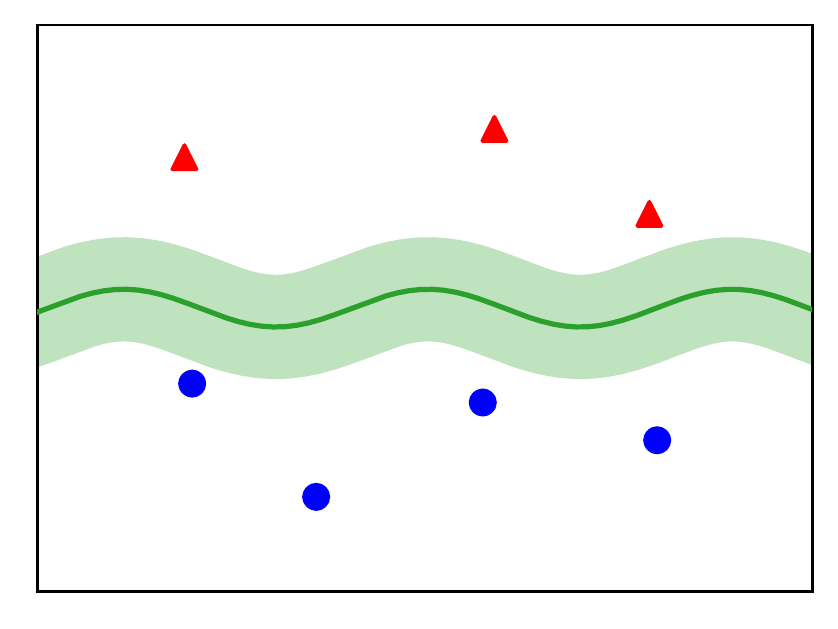}%
\end{subfigure}%
\caption{A bad classification boundary (left) usually has a poor generalization performance and is fragile to adversarial attacks. A better classification boundary (right) keeps itself far from the training examples, and tends to generalize well.}
\label{fig:boundary}
\end{figure}

\begin{figure}[t]
\centering
\begin{subfigure}{0.45\columnwidth}%
\centering%
\includegraphics[width=0.9\columnwidth]{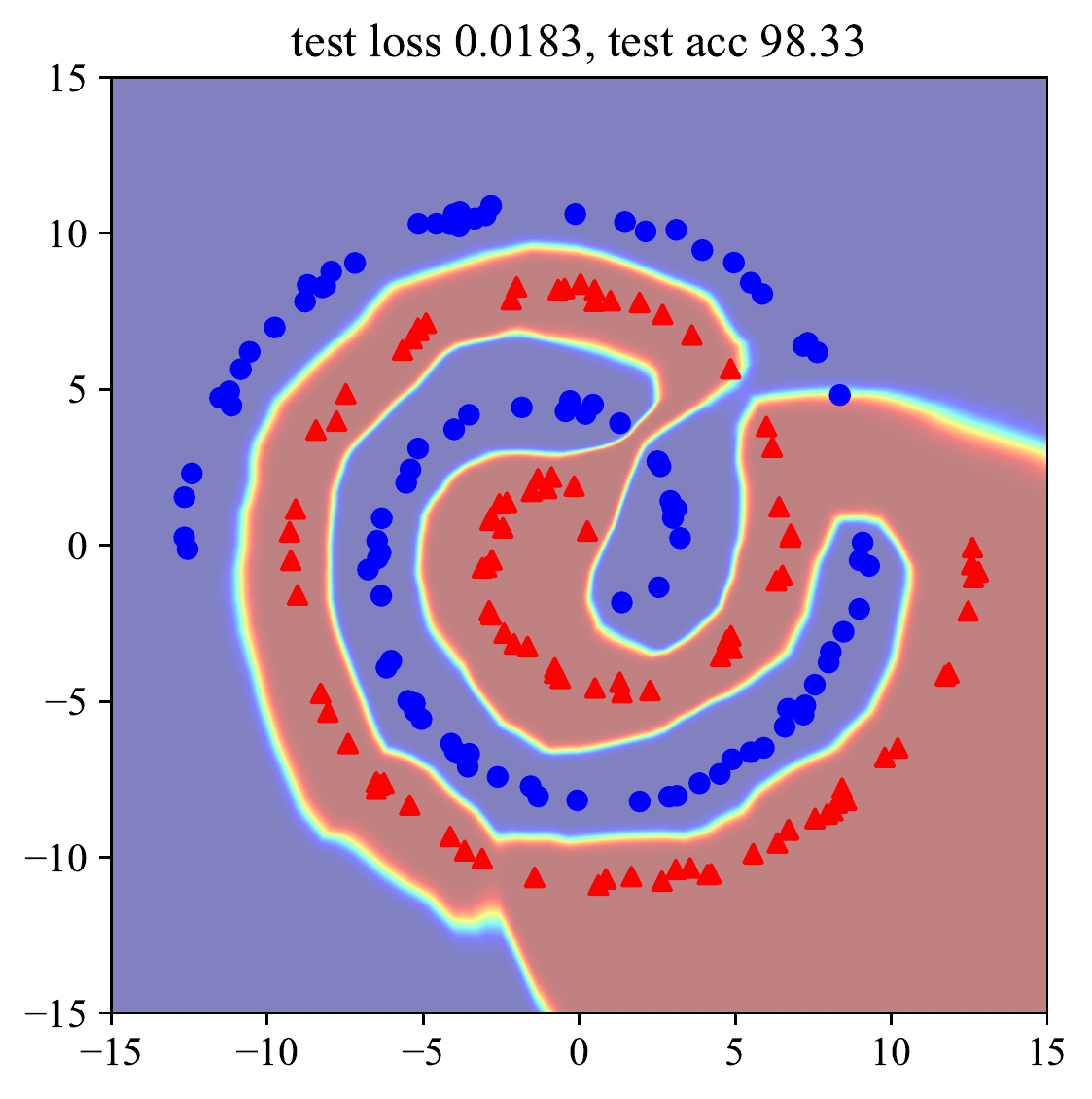}%
\end{subfigure}%
\begin{subfigure}{0.45\columnwidth}%
\centering%
\includegraphics[width=0.9\columnwidth]{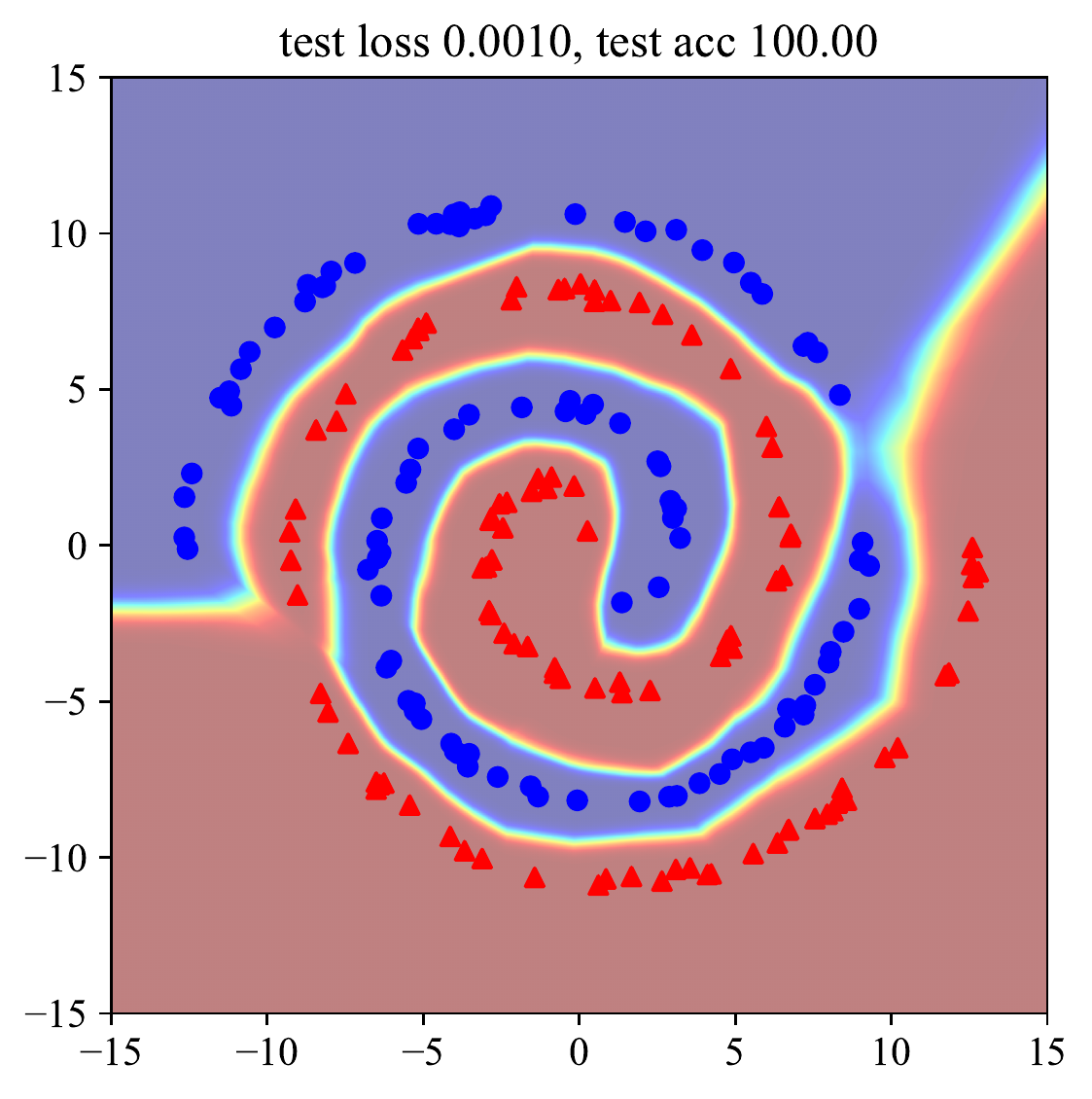}%
\end{subfigure}%
\caption{AMP yields a better classification boundary (right) than ERM does (left). The experiment is conducted on a spiral dataset \cite{sugiyama2015introduction} using a feed-forward network.}
\label{fig:datamanifold}
\end{figure}

The construction and analysis so far have focused on the parameter space. Further insights may be obtained by inspecting AMP in the input space. 

It is evident that the effect of parameter $\boldsymbol{\theta}$ on the input space is no more than defining the class boundaries and specifying how the loss $\ell(\boldsymbol{x}, \boldsymbol{y})$ changes with input $\boldsymbol{x}$. Then minimizing the AMP loss can be seen as finding a class boundary (and loss function $\ell(\boldsymbol{x}, \boldsymbol{y})$) which has the lowest average loss over the training examples, even when the worst $\epsilon$-bounded perturbation is applied. A consequence of such a minimization is arguably creating smoother class boundaries and keeping the training examples not too close to the boundaries. This is illustrated in Figure~\ref{fig:boundary} and experimentally validated in Figure~\ref{fig:datamanifold}. 

When viewed from input space, the adversarial perturbation of the model parameter in AMP shares some similarity with adversarial training \cite{goodfellow2015explaining}. There is also a significant difference between the two: adversarial training defends the model against adversarial attacks whereas AMP defends it against overfitting. We further elaborate why AMP is different from adversarial training in Appendix~\ref{app:c}.

\begin{table*}[t]
\centering
\subcaptionbox{SVHN\label{tab:svhn}}[.66\columnwidth]{%
\resizebox{.66\columnwidth}{!}{%
\begin{tabular}{lcc}
\toprule
PreActResNet18 & Test Error (\%) & Test NLL \\
\midrule
ERM & 2.95$\pm$0.063 & 0.166$\pm$0.004 \\
Dropout & 2.80$\pm$0.065 & 0.156$\pm$0.012 \\
Label Smoothing & 2.78$\pm$0.087 & 0.998$\pm$0.002 \\
Flooding & 2.84$\pm$0.047 & \underline{0.130$\pm$0.003} \\
MixUp & \underline{2.74$\pm$0.044} & 0.146$\pm$0.004 \\
Adv. Training & 2.77$\pm$0.080 & 0.151$\pm$0.018 \\
RMP & 2.93$\pm$0.066 & 0.161$\pm$0.010 \\
AMP & \textbf{2.30$\pm$0.025} & \textbf{0.096$\pm$0.002} \\
\midrule
VGG16 & Test Error (\%) & Test NLL \\
\midrule
ERM & 3.14$\pm$0.060 & 0.140$\pm$0.027 \\
Dropout & 2.96$\pm$0.049 & 0.134$\pm$0.027 \\
Label Smoothing & 3.07$\pm$0.070 & 1.004$\pm$0.002 \\
Flooding & 3.15$\pm$0.085 & 0.128$\pm$0.003 \\
MixUp & 3.09$\pm$0.057 & 0.160$\pm$0.003 \\
Adv. Training & \underline{2.94$\pm$0.091} & \underline{0.122$\pm$0.003} \\
RMP & 3.19$\pm$0.052 & 0.134$\pm$0.004 \\
AMP & \textbf{2.73$\pm$0.015} & \textbf{0.116$\pm$0.006} \\
\bottomrule
\end{tabular}
}%
}%
\subcaptionbox{CIFAR-10\label{tab:cifar10}}[.66\columnwidth]{%
\resizebox{.66\columnwidth}{!}{%
\begin{tabular}{lcc}
\toprule
PreActResNet18 & Test Error (\%) & Test NLL \\
\midrule
ERM & 5.02$\pm$0.212 & 0.239$\pm$0.009 \\
Dropout & 4.86$\pm$0.148 & 0.223$\pm$0.009 \\
Label Smoothing & 4.85$\pm$0.115 & 1.038$\pm$0.003 \\
Flooding & 4.97$\pm$0.082 & \underline{0.166$\pm$0.003} \\
MixUp & \underline{4.09$\pm$0.117} & 0.198$\pm$0.004 \\
Adv. Training & 4.99$\pm$0.085 & 0.247$\pm$0.006 \\
RMP & 4.97$\pm$0.167 & 0.239$\pm$0.008 \\
AMP & \textbf{3.97$\pm$0.091} & \textbf{0.129$\pm$0.003} \\
\midrule
VGG16 & Test Error (\%) & Test NLL \\
\midrule
ERM & 6.32$\pm$0.193 & 0.361$\pm$0.012 \\
Dropout & 6.22$\pm$0.147 & 0.314$\pm$0.009 \\
Label Smoothing & 6.29$\pm$0.158 & 1.076$\pm$0.003 \\
Flooding & 6.26$\pm$0.145 & \underline{0.234$\pm$0.005} \\
MixUp & \textbf{5.48$\pm$0.112} & 0.251$\pm$0.003 \\
Adv. Training & 6.49$\pm$0.130 & 0.380$\pm$0.010 \\
RMP & 6.30$\pm$0.109 & 0.363$\pm$0.010 \\
AMP & \underline{5.65$\pm$0.147} & \textbf{0.207$\pm$0.005} \\
\bottomrule
\end{tabular}
}%
}%
\subcaptionbox{CIFAR-100\label{tab:cifar100}}[.66\columnwidth]{%
\resizebox{.66\columnwidth}{!}{%
\begin{tabular}{lcc}
\toprule
PreActResNet18 & Test Error (\%) & Test NLL \\
\midrule
ERM & 24.31$\pm$0.303 & 1.056$\pm$0.013 \\
Dropout & 24.48$\pm$0.351 & 1.110$\pm$0.021 \\
Label Smoothing & 22.07$\pm$0.256 & 2.099$\pm$0.005 \\
Flooding & 24.50$\pm$0.234 & 0.950$\pm$0.011 \\
MixUp & \underline{21.78$\pm$0.210} & \underline{0.910$\pm$0.007} \\
Adv. Training & 25.23$\pm$0.229 & 1.110$\pm$0.012 \\
RMP & 24.28$\pm$0.138 & 1.059$\pm$0.011 \\
AMP & \textbf{21.51$\pm$0.308} & \textbf{0.774$\pm$0.016} \\
\midrule
VGG16 & Test Error (\%) & Test NLL \\
\midrule
ERM & 27.84$\pm$0.297 & 1.827$\pm$0.209 \\
Dropout & 27.72$\pm$0.337 & 1.605$\pm$0.062 \\
Label Smoothing & 27.49$\pm$0.179 & 2.310$\pm$0.005 \\
Flooding & 27.93$\pm$0.271 & 1.221$\pm$0.037 \\
MixUp & \underline{26.81$\pm$0.254} & \underline{1.136$\pm$0.013} \\
Adv. Training & 29.12$\pm$0.145 & 1.535$\pm$0.389 \\
RMP & 27.81$\pm$0.327 & 1.873$\pm$0.035 \\
AMP & \textbf{25.60$\pm$0.168} & \textbf{1.049$\pm$0.049} \\
\bottomrule
\end{tabular}
}%
}%
\caption{Top-1 classification errors and test neg-log-likelihoods on (a) SVHN, (b) CIFAR-10 and (c) CIFAR-100. We run experiments 10 times to report the mean and the standard deviation of errors and neg-log-likelihoods.}
\label{tab:cvresults}
\end{table*}

\section{Experiment}

We empirically investigate the performance of AMP in various perspectives. Firstly, we compare the generalization ability of AMP on benchmark image classification datasets with several popular regularization schemes in Section \ref{sec:compare}, including Dropout \cite{srivastava2014dropout}, label smoothing \cite{szegedy2016rethinking}, Flooding \cite{ishida2020we}, MixUp \cite{zhang2018mixup} and adversarial training \cite{goodfellow2015explaining}. We also compare our scheme with ERM \cite{vapnik1998statistical}, which does not utilize any regularization and optimizes the neural network with $\mathcal{L}_\mathrm{ERM}$. We include a baseline named random model perturbation (RMP) for comparison. Specifically, RMP applies a random perturbation (instead of the ``worst perturbation'') to the parameter within a small range to help the model to find a better minimum. Then, we study the performance of AMP on more complex deep architectures with powerful data augmentation techniques in Section~\ref{sec:augment}. In addition, we investigate the calibration effect of AMP in Section~\ref{sec:calibration} and demonstrate the influence of perturbation in Section~\ref{sec:influence}. Finally, we compare the computational cost of AMP with ERM in Section~\ref{sec:cost}. The implementation is on PyTorch framework \cite{paszke2019pytorch}, and the experiments are carried out on NVIDIA Tesla V100 GPUs.

\subsection{Comparison of the Generalization Ability}\label{sec:compare}

Three publicly available benchmark image datasets are used for performance evaluation. The SVHN dataset \cite{netzer2011reading} has 10 classes containing 73,257 digits for training and 26,032 digits for testing. Limited by the computing resource, we did not use the additional 531,131 images in SVHN training. The CIFAR datasets contain 32$\times$32-pixel colour images, where CIFAR-10 has 10 classes containing 5,000 images for training and 1,000 images for testing per class, CIFAR-100 has 100 classes containing 500 images for training and 100 images for testing per class \cite{krizhevsky2009learning}. 

Two representative deep architectures for image classification, PreActResNet18 \cite{he2016identity} and VGG16 \cite{simonyan2015very}, are taken as the underlying classifier. In the training procedure, random crops and horizontal flips are adopted as data augmentation schemes. We compute the mean and standard derivation on the training set to normalize the input images. SGD with momentum is exploited as the optimizer with a step-wise learning rate decay. Specifically, the outer learning rate is initialized as $0.1$ and divided by 10 after 100 and 150 epochs. We train each model for 200 epochs on the training set with 50 examples per mini-batch. Weight decay is set to $10^{-4}$ for all compared models. We tune hyper-parameters on each dataset using 10\% of the training set as the validation set. For Dropout, we randomly choose 10\% of the neurons in each layer after ReLU activation and deactivate them at each training iteration. The label smoothing coefficient is set to $0.2$ and the flooding level is set to $0.02$. For MixUp, we follow the original study \cite{zhang2018mixup} and linearly combine random pairs of training examples by using coefficient variables drawn from $\text{Beta}(1,1)$. For adversarial training, we set the perturbation size to 1 for each pixel and take one single step to generate adversarial examples. For RMP, we set the $\text{L}_2$ norm ball radius to $0.1$. For AMP, we fix the number of inner iteration as $N=1$, and adopt $\epsilon=0.5,\gamma=1$ for PreActResNet18 and $\epsilon=0.1,\gamma=0.2$ for VGG16. Top-1 classification error and test neg-log-likelihood are reported in Table~\ref{tab:cvresults}.

From Table~\ref{tab:cvresults}, it is evident that AMP outperforms the baseline methods in various settings, in terms of both classification error and test neg-log-likelihood, except on the CIFAR-10 dataset where VGG16 is employed. Despite the remarkable performance of AMP, MixUp also demonstrates competitive improvement in classification accuracy, and Flooding achieves small testing neg-log-likelihood in many settings. We note that, compared with AMP, the results of RMP suggest that randomly perturbing parameters cannot obtain comparable performance to the AMP. This result confirms that the adversarial perturbation provides the most useful information to the regularizer. Above results describe the efficiency of AMP in regularizing deep networks.

\subsection{Improvement over Data Augmentation}\label{sec:augment}

Data augmentation techniques can be viewed as regularization schemes since through introducing additional training examples, they impose additional constraints on the model parameter thereby improving generalization. To validate the effectiveness of AMP over other data augmentation techniques, we choose vanilla augmentation \cite{krizhevsky2012imagenet}, Cutout \cite{devries2017improved} and AutoAugment \cite{cubuk2019autoaugment} as underlying augmentation methods and compare the classification accuracy of AMP with ERM. The vanilla augmentation exploits manually designed policies including random crops and horizontal flips. We use the same Cutout configuration and AutoAugment policy as their corresponding original studies. For the hyper-parameters of AMP, we fix $N=1$ and adopt $\epsilon=0.5,\gamma=1$ for vanilla augmentation, $\epsilon=0.3,\gamma=0.5$ for Cutout, and $\epsilon=0.1,\gamma=0.1$ for AutoAugment. We employ two recent powerful deep architectures, WideResNet \cite{zagoruyko2016wide} and PyramidNet \cite{han2017pyramid}, with the compared data augmentation techniques on SVHN, CIFAR-10 and CIFAR-100. The top-1 classification errors are shown in Table~\ref{tab:cvresults2}. The results suggest AMP's regularization effect in the presence of advanced data augmentation techniques.

\begin{table}[t]
\centering
\resizebox{.99\columnwidth}{!}{%
\begin{tabular}{llcccc}
\toprule
 & & \multicolumn{2}{c}{WideResNet-28-10} & \multicolumn{2}{c}{PyramidNet-164-270}\\
 & & ERM & AMP & ERM & AMP \\
\midrule
 & Vanilla & 2.57$\pm$0.067 & \textbf{2.19$\pm$0.036} & 2.47$\pm$0.034 & \textbf{2.11$\pm$0.041} \\
SVHN & Cutout & 2.27$\pm$0.085 & \textbf{1.83$\pm$0.018} & 2.19$\pm$0.021 & \textbf{1.82$\pm$0.023} \\
 & AutoAug & 1.91$\pm$0.059 & \textbf{1.61$\pm$0.024} & 1.80$\pm$0.044 & \textbf{1.35$\pm$0.056} \\
\midrule
 & Vanilla & 3.87$\pm$0.167 & \textbf{3.00$\pm$0.059} & 3.60$\pm$0.197 & \textbf{2.75$\pm$0.040} \\
CIFAR-10 & Cutout & 3.38$\pm$0.081 & \textbf{2.67$\pm$0.043} & 2.83$\pm$0.102 & \textbf{2.27$\pm$0.034} \\
 & AutoAug & 2.78$\pm$0.134 & \textbf{2.32$\pm$0.097} & 2.49$\pm$0.128 & \textbf{1.98$\pm$0.062} \\
\midrule
 & Vanilla & 19.17$\pm$0.270 & \textbf{17.33$\pm$0.110} & 17.13$\pm$0.210 & \textbf{15.09$\pm$0.092} \\
CIFAR-100 & Cutout & 18.12$\pm$0.114 & \textbf{16.04$\pm$0.071} & 16.45$\pm$0.136 & \textbf{14.34$\pm$0.153} \\
 & AutoAug & 17.79$\pm$0.185 & \textbf{14.95$\pm$0.088} & 15.43$\pm$0.269 & \textbf{13.36$\pm$0.245} \\
\bottomrule
\end{tabular}
}
\caption{Mean and standard deviation of top-1 errors (\%) on SVHN, CIFAR-10 and CIFAR-100 over 10 trials.}
\label{tab:cvresults2}
\end{table}

\subsection{Calibration Effect}\label{sec:calibration}

A well calibrated neural network is one in which the predicted softmax scores give better indicators of the actual likelihood of a correct prediction. Deep neural networks without any regularizer are prone to overconfidence on incorrect predictions, and a well calibrated network is required especially in some application areas like object detection \cite{girshick2015fast,ren2015faster} and autonomous vehicle control \cite{chen2015deepdriving,levinson2011towards}. AMP chooses a flatter minimum of the empirical risk, which gives less confidence on possibly misclassified examples, ensuring the neural network to be better calibrated. On the contrary, the sharp minima given by ERM may contain incorrectly classified examples, which makes overconfident predictions on misclassified examples. In this section, we demonstrate that AMP can improve the calibration effect of neural networks.

We adopt the measurement of calibration described in \cite{guo2017calibration}, namely, the Expected Calibration Error (ECE) (see Appendix~\ref{app:d} for definition). To evaluate the calibration effect of different regularization schemes, we compute ECEs of the pretrained PreActResNet18 models on the SVHN, CIFAR-10 and CIFAR-100 datasets. The results are shown in Figure~\ref{fig:ece}. From the results, AMP achieves excellent calibration performance on various datasets. We can also find that Flooding outperforms other methods in calibration error on the CIFAR-10 dataset. Additionally, label smoothing degrades the calibration effect of neural networks, since it excessively biases the labels to the uniform distribution.

\begin{figure}[t]
\centering
\includegraphics[width=0.92\columnwidth]{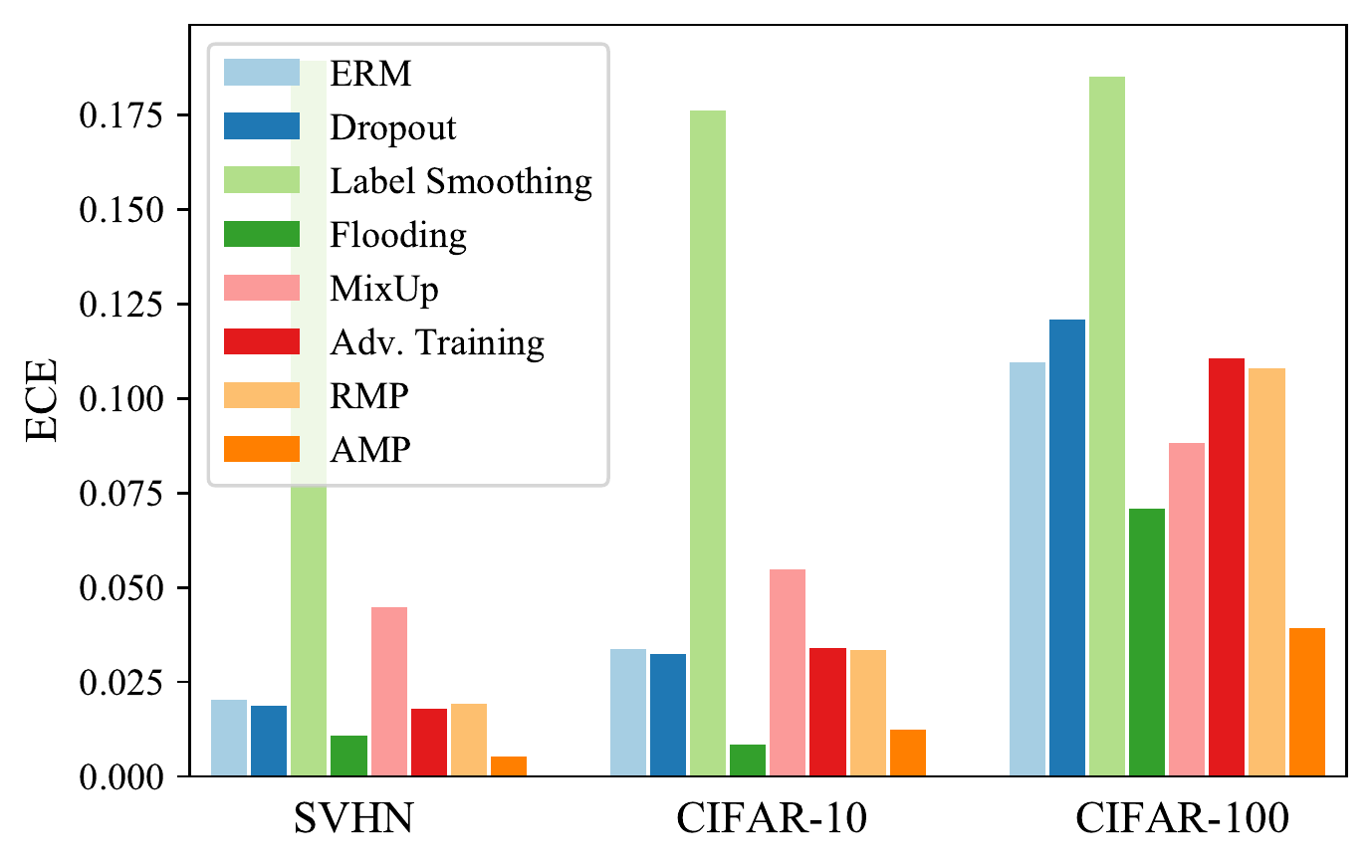}
\caption{Expected calibration errors (ECEs) of AMP and other baseline methods. Results are averaged over 10 trials.}
\label{fig:ece}
\end{figure}

\subsection{Influence of Perturbation}\label{sec:influence}

To investigate the relationship between the perturbation magnitudes and the geometry of the selected local minima, we compare the empirical risks of models trained with ERM and AMP. To clearly illustrate this, we adopt $\eta=2$ and $N=2$. Figure~\ref{fig:tune} shows the empirical risks on CIFAR-10 training and test set by varying the perturbation radius $\epsilon$. It can be seen that $\mathcal{L}_\mathrm{ERM}(\boldsymbol{\theta}^\ast_\mathrm{AMP})$ on the training set tends to be high when radius is large. This indicates that the selected minimum has a smaller depth. Moreover, when evaluating on the test set, $\mathcal{L}_\mathrm{ERM}(\boldsymbol{\theta}^\ast_\mathrm{AMP})$ arrives at minimum when the perturbation radius is around $0.06$. This suggests that an appropriate magnitude of perturbation regularizes networks efficiently, corresponds to the good properties of the selected minimum both in flatness and depth. Results on more datasets are given in Appendix~\ref{app:e}.

\begin{figure}[t]
\centering
\begin{subfigure}{0.49\columnwidth}%
\centering%
\captionsetup{width=0.88\columnwidth}%
\includegraphics[width=0.98\columnwidth]{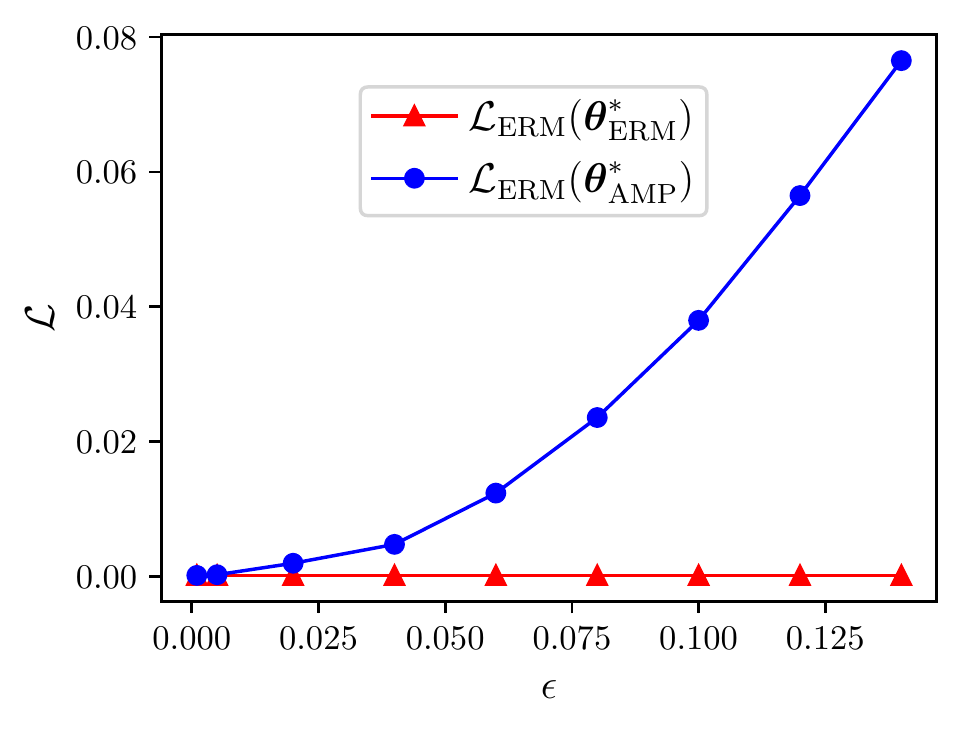}%
\caption{CIFAR-10 Training Set}%
\end{subfigure}%
\begin{subfigure}{0.49\columnwidth}%
\centering%
\captionsetup{width=0.88\columnwidth}%
\includegraphics[width=0.98\columnwidth]{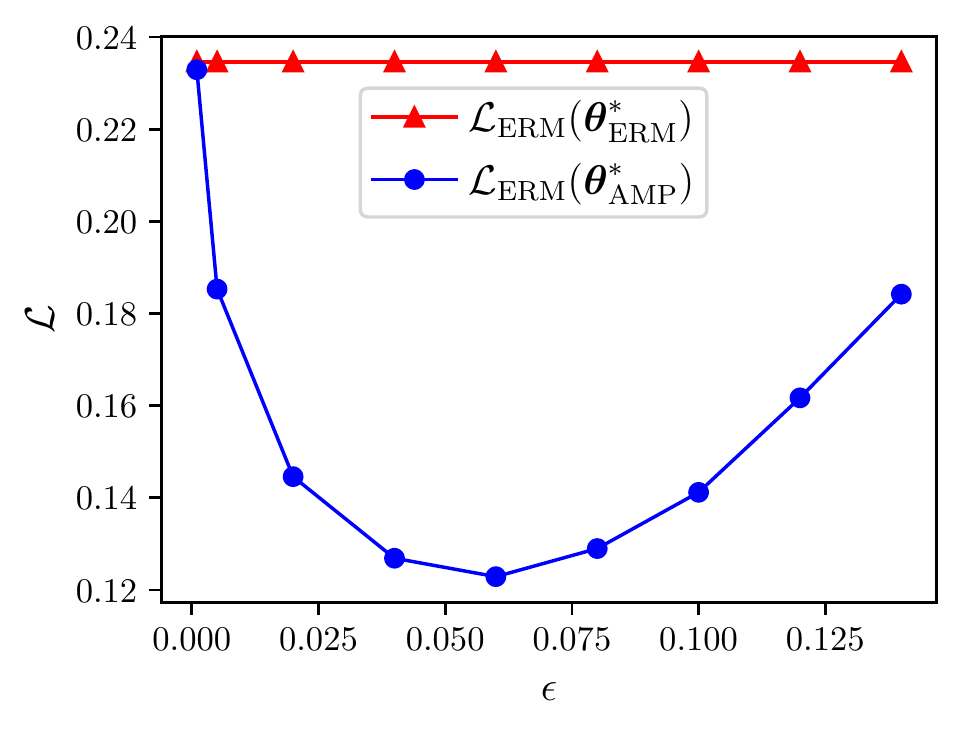}%
\caption{CIFAR-10 Test Set}%
\end{subfigure}%
\caption{The empirical risks of the model trained with ERM (red) and AMP (blue) on the CIFAR-10 training set and test set with varying perturbation radius.}
\label{fig:tune}
\end{figure}

\subsection{Computational Cost}\label{sec:cost}

AMP computes the adversarial perturbation of the model parameter in the training phase, such an operation arguably increases the computational cost. In addition to the gradient computation for updating the parameter, each stochastic gradient descent iteration requires multiple gradient computations to produce the adversarial perturbations. The computational cost will significantly increase as the inner iteration number $N$ grows. However, we find that $N=1$ is sufficient to regularize the neural networks. Under this setting, we evaluate the practical computational cost of AMP compared to the ERM method. From our observation, AMP usually takes around $1.8\times$ that of ERM training. Therefore, the extra effort for adversarial perturbation is affordable.

\section{Conclusion}

Regularization is the main tool for deep learning practitioners to combat overfitting. In this work, we propose a novel regularization scheme, Adversarial Model Perturbation (AMP), built upon the understanding that flat local minima lead to better generalization. Unlike many other data-dependent regularization schemes, which are large of a heuristic nature, AMP has strong theoretical justifications under a certain approximating assumption. These justifications also allow us to predict its behaviour with respect to varying hyper-parameters. Our theoretical analysis and the regularization performance of AMP are confirmed through extensive experiments on image classification datasets. It is also observed that AMP helps the model to better calibrate itself. The outstanding performance of AMP arguably makes it into the current state of the art among all regularization schemes. The empirical validation of AMP presented in this paper appears to further confirm the connection between flat minima and generalization.

\paragraph{Acknowledgements.}
This work is supported partly by the National Key Research and Development Program of China (2016YFB1000103), by the National Natural Science Foundation of China (No. 61772059), by the Beijing Advanced Innovation Center for Big Data and Brain Computing, by the Fundamental Research Funds for the Central Universities, by the Beijing S\&T Committee (No. Z191100008619007) and by the State Key Laboratory of Software Development Environment (No. SKLSDE-2020ZX-14). The authors specially thank Linfang Hou for helpful discussions.

{\small
\bibliographystyle{ieee_fullname}
\bibliography{main}
}

\clearpage

\section*{Appendix}

\renewcommand{\thesubsection}{\Alph{subsection}}

\subsection{Proof of Corollary~\ref{cor2}}\label{app:a}

\begin{innercor}{2}
Suppose that $C_1=C_2$. Let $A_2=\beta A_1$ for some $\beta<1$. Note that in this setting, $\gamma^\ast_1<\gamma^\ast_2$. Suppose that $\sigma_2^2=r\sigma^2_1$ for some positive $r$. Then 
\begin{equation*}
\gamma^\ast_{2,\mathrm{AMP}}<\gamma^\ast_{1,\mathrm{AMP}}
\end{equation*}
if and only if
\begin{equation*}
\beta > \exp\left(-\frac{\epsilon^2}{2\sigma^2_1}\right) ~ {\rm and} ~
r > \frac{1}{1+\frac{2\sigma_1^2}{\epsilon^2}\log\beta}
\end{equation*}
\end{innercor}

\begin{proof}
Since $C_1=C_2$, we have 
\begin{equation*}
0 > \exp\left(-\frac{\epsilon^2}{2\sigma_1^2}\right)
- \beta\exp\left(-\frac{\epsilon^2}{2r\sigma_1^2}\right)
\end{equation*}

It then follows that
\begin{equation}\label{eq:rAndBeta}
\frac{\epsilon^2}{2r\sigma_1^2}<\frac{\epsilon^2}{2\sigma_1^2}+\log\beta
\end{equation}
noting that $\log\beta<0$, we have $r>1$.  

Further manipulating (\ref{eq:rAndBeta}), we get
\begin{equation}\label{eq:rAndBeta2}
\frac{1}{r}<1+\frac{2\sigma_1^2}{\epsilon^2}\log\beta
\end{equation}

Since $r>0$, the right side of (\ref{eq:rAndBeta2}) is positive, which gives rise to 
\begin{equation*}
\beta>\exp\left(-\frac{\epsilon^2}{2\sigma^2_1}\right)
\end{equation*}

Continuing with (\ref{eq:rAndBeta2}), we arrive at
\begin{equation*}
r>\frac{1}{1+\frac{2\sigma_1^2}{\epsilon^2}\log\beta}
\end{equation*}

This proves the result.
\end{proof}

\subsection{Proof of Theorem \ref{thm2}}\label{app:b}

\begin{innerthm}{2}
Let $N=1$. Then for a sufficiently small inner learning rate $\zeta$, a minimization update step in AMP training for batch $\mathcal{B}$ is equivalent to a gradient-descent step on the following loss function with learning rate $\eta$:
\begin{equation*}
\widetilde{\mathcal{J}}_{\mathrm{ERM},\mathcal{B}}(\boldsymbol{\theta}):=\mathcal{J}_{\mathrm{ERM},\mathcal{B}}(\boldsymbol{\theta})+\Omega(\boldsymbol{\theta})
\end{equation*}
where
\begin{equation*}
\Omega(\boldsymbol{\theta}):=\begin{cases}
\zeta\Vert\nabla_{\boldsymbol{\theta}}\mathcal{J}_{\mathrm{ERM},\mathcal{B}}(\boldsymbol{\theta})\Vert_2^2,&\Vert\zeta\nabla_{\boldsymbol{\theta}}\mathcal{J}_{\mathrm{ERM},\mathcal{B}}(\boldsymbol{\theta})\Vert_2\le\epsilon\\
\epsilon\Vert\nabla_{\boldsymbol{\theta}}\mathcal{J}_{\mathrm{ERM},\mathcal{B}}(\boldsymbol{\theta})\Vert_2,&\Vert\zeta\nabla_{\boldsymbol{\theta}}\mathcal{J}_{\mathrm{ERM},\mathcal{B}}(\boldsymbol{\theta})\Vert_2>\epsilon
\end{cases}
\end{equation*}
\end{innerthm}

\begin{proof}
At each training step of AMP, we adversarially perturb the parameter with a step size of $\zeta$. If the norm of perturbation is larger than a preset value $\epsilon$, it will be projected onto the $\text{L}_2$-norm ball. Denoted by $\boldsymbol{\theta}_k$ the model parameter at the $k$-th iteration, the perturbed parameter is:
\begin{equation*}
\boldsymbol{\theta}_{k,\mathrm{adv}}\!\!=\!\!\begin{cases}
\boldsymbol{\theta}_k\!+\!\zeta\nabla_{\boldsymbol{\theta}}\mathcal{J}_{\mathrm{ERM},\mathcal{B}}(\boldsymbol{\theta}_k),\!\!&\!\!\!\Vert\zeta\nabla_{\boldsymbol{\theta}}\mathcal{J}_{\mathrm{ERM},\mathcal{B}}(\boldsymbol{\theta}_k)\Vert_2\!\le\!\epsilon\\
\boldsymbol{\theta}_k\!+\!\epsilon\frac{\nabla_{\boldsymbol{\theta}}\mathcal{J}_{\mathrm{ERM},\mathcal{B}}(\boldsymbol{\theta}_k)}{\Vert\nabla_{\boldsymbol{\theta}}\mathcal{J}_{\mathrm{ERM},\mathcal{B}}(\boldsymbol{\theta}_k)\Vert_2},\!\!&\!\!\!\Vert\zeta\nabla_{\boldsymbol{\theta}}\mathcal{J}_{\mathrm{ERM},\mathcal{B}}(\boldsymbol{\theta}_k)\Vert_2\!>\!\epsilon\\
\end{cases}
\end{equation*}

Then the parameter is updated according to the gradient computed by the perturbed parameter with a step size of $\eta$:
\begin{equation*}
\boldsymbol{\theta}_{k+1}=\boldsymbol{\theta}_k-\eta\nabla_{\boldsymbol{\theta}}\mathcal{J}_{\mathrm{ERM},\mathcal{B}}(\boldsymbol{\theta}_{k,\mathrm{adv}})
\end{equation*}

With a sufficient small $\zeta$, we can utilize the first-order Taylor expansion $f(\boldsymbol{x}+\boldsymbol{\delta})\approx f(\boldsymbol{x})+\boldsymbol{\delta}^T\nabla_{\boldsymbol{x}}f(\boldsymbol{x})$. In the former condition (i.e. $\Vert\zeta\nabla_{\boldsymbol{\theta}}\mathcal{J}_{\mathrm{ERM},\mathcal{B}}(\boldsymbol{\theta})\Vert_2\le\epsilon$), we have:
\begin{align*}
\boldsymbol{\theta}_{k+1}&=\boldsymbol{\theta}_k-\eta\nabla_{\boldsymbol{\theta}}\mathcal{J}_{\mathrm{ERM},\mathcal{B}}\left(\boldsymbol{\theta}_k+\zeta\nabla_{\boldsymbol{\theta}}\mathcal{J}_{\mathrm{ERM},\mathcal{B}}(\boldsymbol{\theta}_k)\right)\nonumber\\
&\approx\boldsymbol{\theta}_k-\eta\nabla_{\boldsymbol{\theta}}\left(\mathcal{J}_{\mathrm{ERM},\mathcal{B}}(\boldsymbol{\theta}_k)+\zeta\Vert\nabla_{\boldsymbol{\theta}}\mathcal{J}_{\mathrm{ERM},\mathcal{B}}(\boldsymbol{\theta}_k)\Vert_2^2\right)
\end{align*}

In the latter condition (i.e. $\Vert\zeta\nabla_{\boldsymbol{\theta}}\mathcal{J}_{\mathrm{ERM},\mathcal{B}}(\boldsymbol{\theta})\Vert_2>\epsilon$), we have:
\begin{align*}
\boldsymbol{\theta}_{k+1}&=\boldsymbol{\theta}_k-\eta\nabla_{\boldsymbol{\theta}}\mathcal{J}_{\mathrm{ERM},\mathcal{B}}\left(\boldsymbol{\theta}_k+\epsilon\frac{\nabla_{\boldsymbol{\theta}}\mathcal{J}_{\mathrm{ERM},\mathcal{B}}(\boldsymbol{\theta}_k)}{\Vert\nabla_{\boldsymbol{\theta}}\mathcal{J}_{\mathrm{ERM},\mathcal{B}}(\boldsymbol{\theta}_k)\Vert_2}\right)\nonumber\\
&\approx\boldsymbol{\theta}_k-\eta\nabla_{\boldsymbol{\theta}}\left(\mathcal{J}_{\mathrm{ERM},\mathcal{B}}(\boldsymbol{\theta}_k)+\epsilon\Vert\nabla_{\boldsymbol{\theta}}\mathcal{J}_{\mathrm{ERM},\mathcal{B}}(\boldsymbol{\theta}_k)\Vert_2\right)
\end{align*}

This proves the theorem.
\end{proof}

\subsection{Why AMP is not Adversarial Training}\label{app:c}

In this section, we will further discuss the difference between AMP and adversarial training (ADV).

It is sensible that perturbing weights $\boldsymbol{\theta}$ may have an effect similar to perturbing the examples $\boldsymbol{x}$ since $\boldsymbol{\theta}$ and $\boldsymbol{x}$ usually appear together via inner product $\boldsymbol{\theta}^\mathsf{T}\boldsymbol{x}$. However we note that except for some peculiar cases (such as linear network with some peculiar choices of the loss function or a set of peculiarly constructed training examples), in general the solution $\boldsymbol{\theta}^*_{\rm AMP}$ to the AMP optimization problem is different from the solution $\boldsymbol{\theta}^*_{\rm ADV}$ to the ADV counterpart. The difference between $\boldsymbol{\theta}^*_{\rm AMP}$ and $\boldsymbol{\theta}^*_{\rm ADV}$ can be attributed to two sources. 

First, let $\ell(\boldsymbol{x};\boldsymbol{\theta})$ denote the ERM loss for a single training example $\boldsymbol{x}$. For $N$ examples, the overall ERM loss $\mathcal{L}_{\rm ERM}$ is the sum (or average) of $\ell(\boldsymbol{x}_i;\boldsymbol{\theta})$ over all examples $\boldsymbol{x}_i$, $i=1,\ldots,N$. In AMP, the perturbation is to maximize the {\em overall} empirical loss $\mathcal{L}_{\rm ERM}$ and this perturbation is applied {\em globally} to weights $\boldsymbol{\theta}$. However, in ADV, the perturbation is applied {\em individually} to {\em each} training example $x_i$, with the objective of maximizing the {\em individual} ERM loss $\ell(\boldsymbol{x}_i;\boldsymbol{\theta})$.

Second, even in the case when there is only one training example $\boldsymbol{x}$ so that $\mathcal{L}_{\rm ERM}=\ell$, $\boldsymbol{\theta}^*_{\rm AMP}$ and $\boldsymbol{\theta}^*_{\rm ADV}$ may still be different. Here is an example. Let 
\begin{equation*}
g(z):=\begin{cases}
z   & {\rm if}~z\ge0\\
-2z & {\rm if}~z<0
\end{cases}
\end{equation*}

\begin{figure}[t]
\centering
\includegraphics[width=0.73\columnwidth]{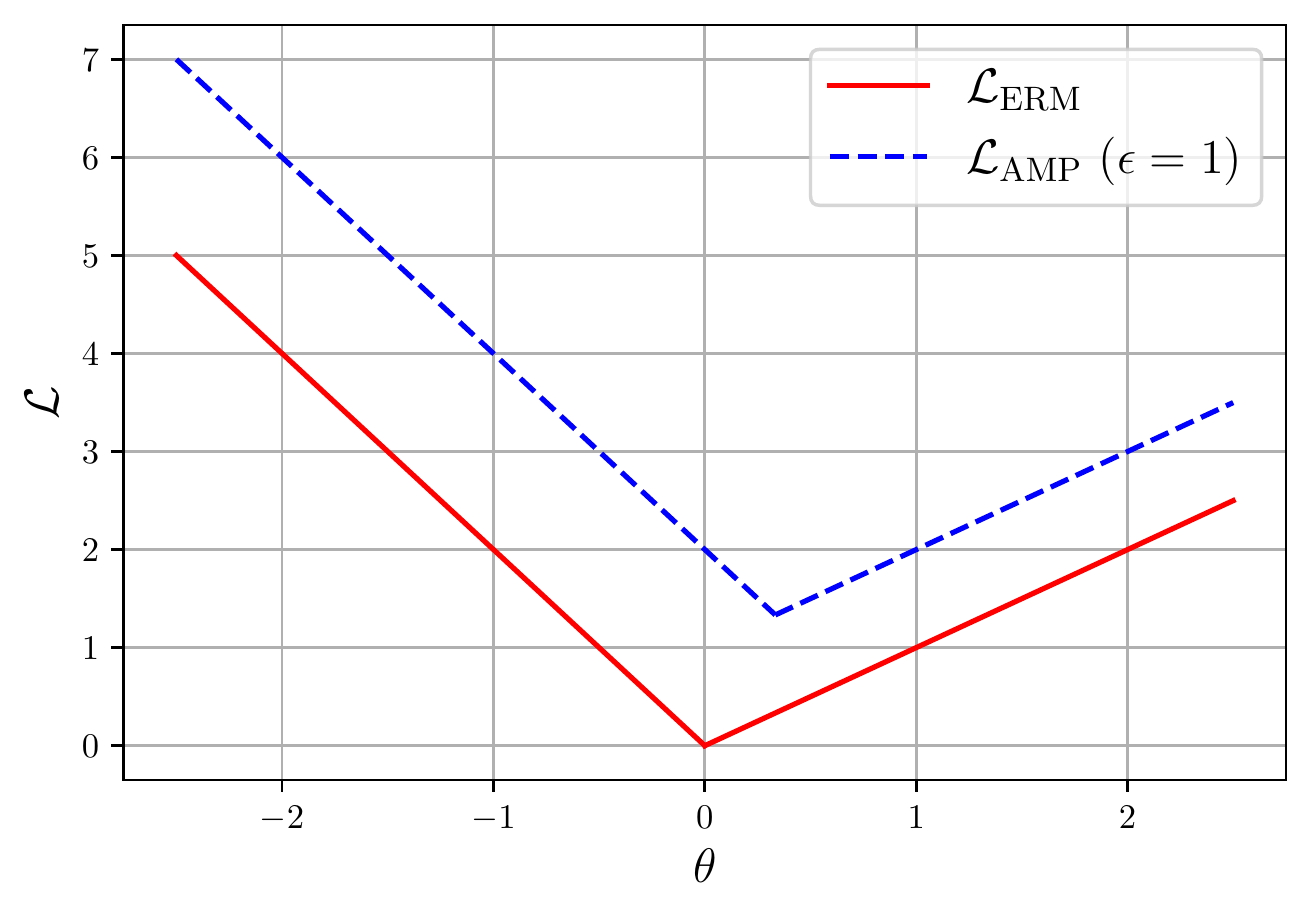}
\caption{The losses of ERM and AMP with varying $\theta$.}
\label{fig:diff}
\end{figure}

Consider that there is a single scalar example $x=1$ and  the weight $\theta$ is a scalar. Define $\ell(x; \theta)=g(\theta x)$.  It can be verified that $\theta^*_{\rm ADV}=\theta^*_{\rm ERM}=0$ regardless of the perturbation radius $\epsilon$, but $\theta^*_{\rm AMP}=\epsilon/3$ (see Figure~\ref{fig:diff}, where the losses are plotted as functions of $\theta$).

\subsection{Definition of Expected Calibration Error}\label{app:d}

We follow the definition presented in the previous work \cite{guo2017calibration}. Firstly, the predictions are grouped into $M$ interval bins of equal sizes. Let $B_m$ be the set of indices of samples whose prediction scores (the winning softmax score) fall into the interval $I_m=(\frac{m-1}{M},\frac mM]$. The accuracy and confidence of $B_m$ are defined as:
\begin{align*}
\text{acc}(B_m)&=\frac{1}{|B_m|}\sum_{i\in B_m}\mathbf{1}(\hat{y}_i=y_i)\\
\text{conf}(B_m)&=\frac{1}{|B_m|}\sum_{i\in B_m}\hat{p}_i
\end{align*}
where $\hat{y}_i$ and $y_i$ are the predicted label and true class labels for sample $i$, $\hat{p}_i$ is the confidence (the winning softmax score) of sample $i$. The {\em Expected Calibration Error} (ECE) is defined as the difference in expectation between confidence and accuracy, i.e.:
\begin{equation*}
\text{ECE}=\sum_{m=1}^M\frac{|B_m|}{n}\bigg\vert\text{acc}(B_m)-\text{conf}(B_m)\bigg\vert
\end{equation*}
where $n$ is the number of samples.

\subsection{Influence of Perturbation}\label{app:e}

We plot the empirical risks of the pretrained PreActResNet18 models on three image datasets with varying perturbation radius in Figure~\ref{fig:tune_full}. To clearly illustrate this, we adopt $\eta=2$ and $N=2$. In these experiments, the perturbation radius $\epsilon$ meets the sweet spots around $0.06$ on all the three datasets, where $\mathcal{L}_\mathrm{ERM}(\boldsymbol{\theta}_\mathrm{AMP}^\ast)$ gets the minimum value. 

\begin{figure*}[t]
\centering
\begin{subfigure}{0.66\columnwidth}%
\centering%
\captionsetup{width=0.9\columnwidth}%
\includegraphics[width=0.95\columnwidth]{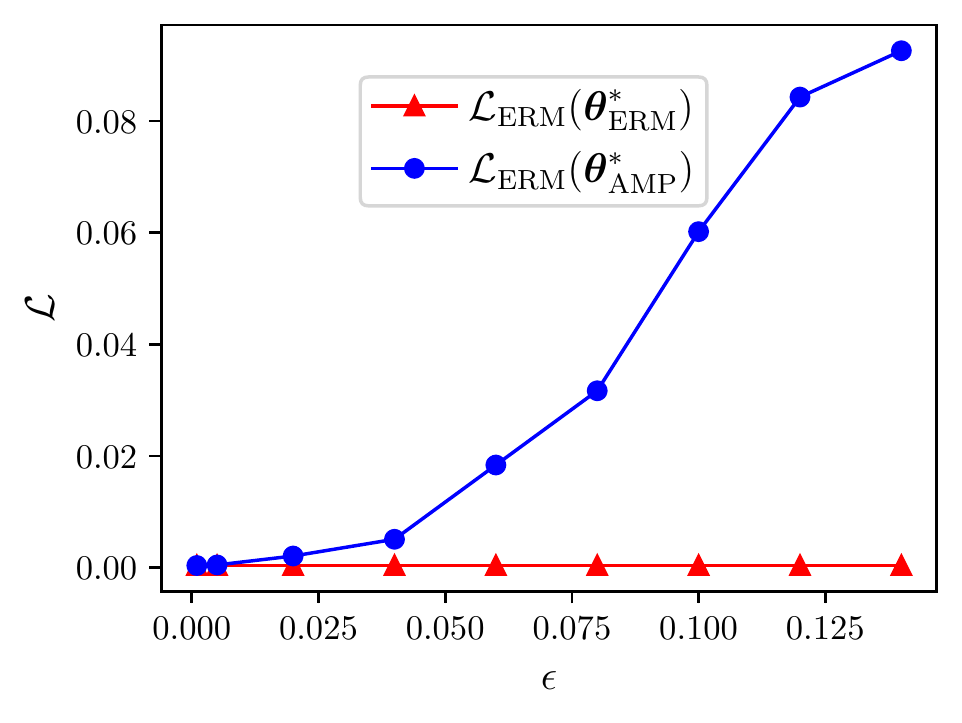}%
\caption{SVHN Training Set}%
\end{subfigure}%
\begin{subfigure}{0.66\columnwidth}%
\centering%
\captionsetup{width=0.9\columnwidth}%
\includegraphics[width=0.95\columnwidth]{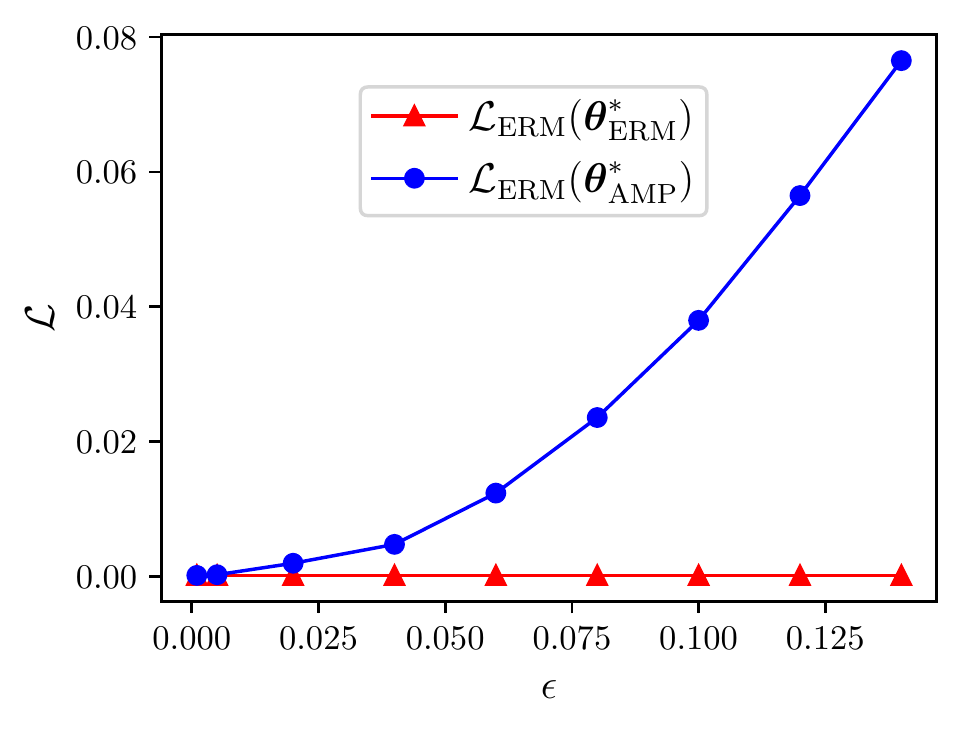}%
\caption{CIFAR-10 Training Set}%
\end{subfigure}%
\begin{subfigure}{0.66\columnwidth}%
\centering%
\captionsetup{width=0.9\columnwidth}%
\includegraphics[width=0.95\columnwidth]{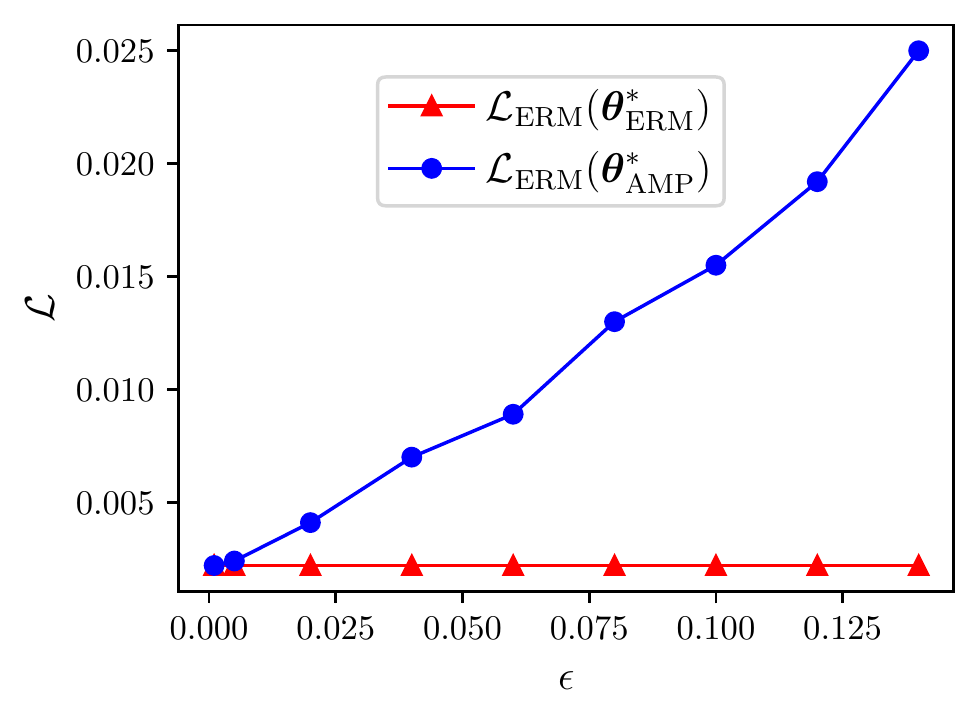}%
\caption{CIFAR-100 Training Set}%
\end{subfigure}%
\\
\begin{subfigure}{0.66\columnwidth}%
\centering%
\captionsetup{width=0.9\columnwidth}%
\includegraphics[width=0.95\columnwidth]{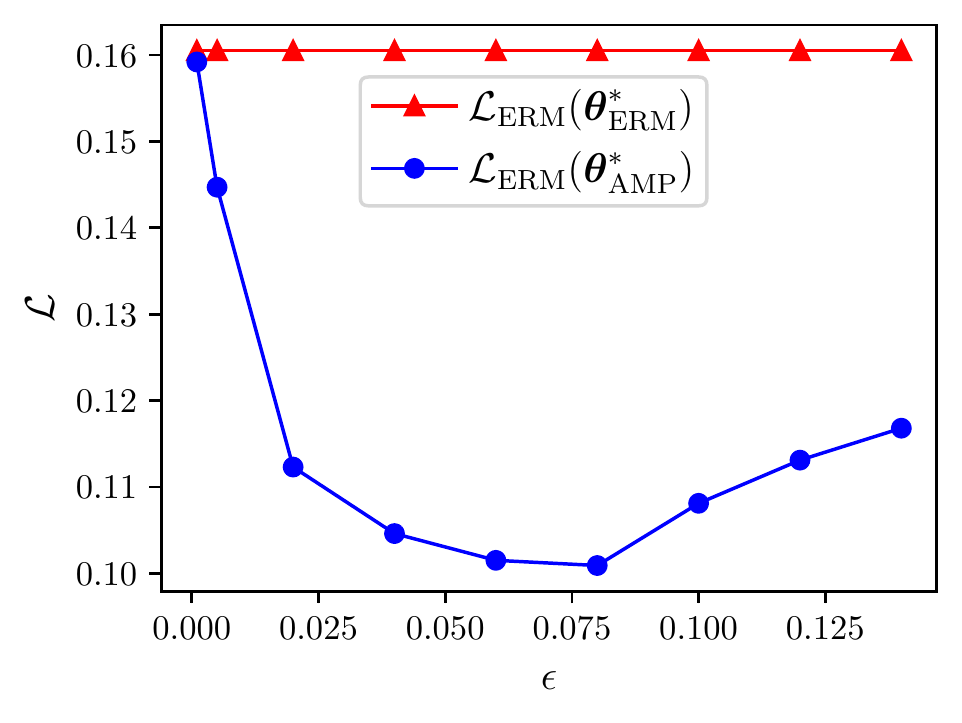}%
\caption{SVHN Test Set}%
\end{subfigure}%
\begin{subfigure}{0.66\columnwidth}%
\centering%
\captionsetup{width=0.9\columnwidth}%
\includegraphics[width=0.95\columnwidth]{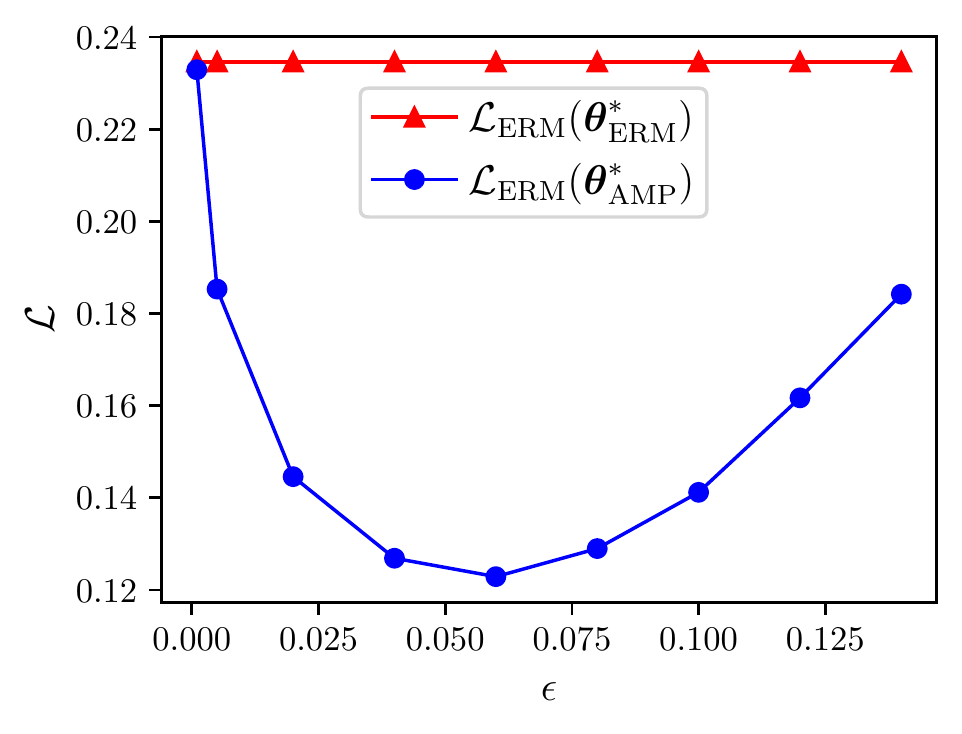}%
\caption{CIFAR-10 Test Set}%
\end{subfigure}%
\begin{subfigure}{0.66\columnwidth}%
\centering%
\captionsetup{width=0.9\columnwidth}%
\includegraphics[width=0.95\columnwidth]{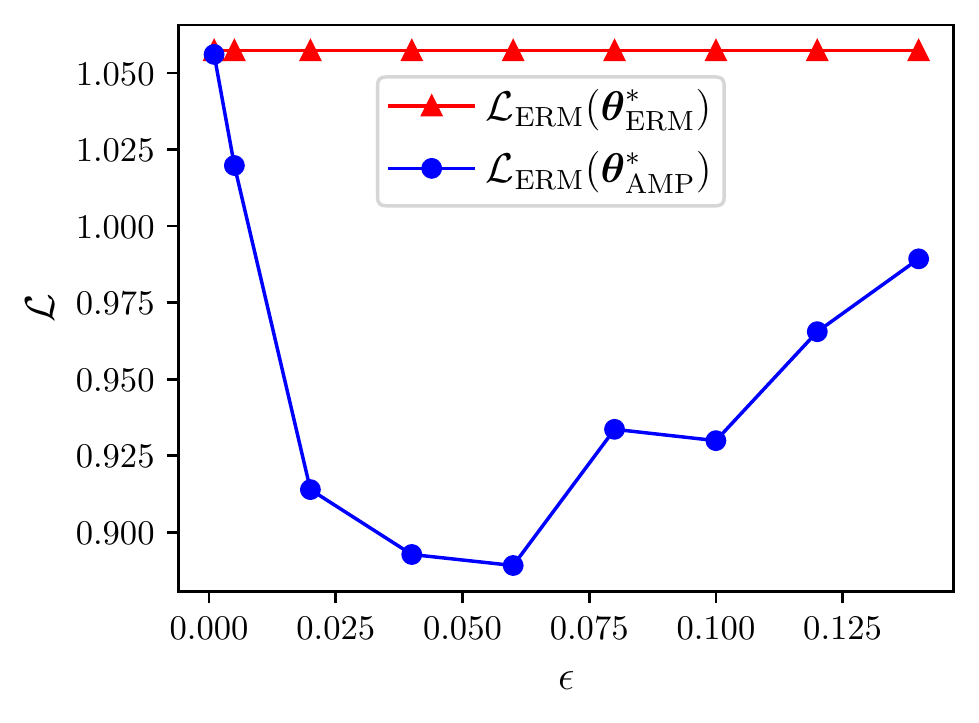}%
\caption{CIFAR-100 Test Set}%
\end{subfigure}%
\caption{The comparison of $\mathcal{L}_{\rm ERM}$ of the models trained with ERM (red) and AMP (blue) with varying perturbation radius.}
\label{fig:tune_full}
\end{figure*}

\subsection{Robustness to Adversarial Attacks}

\begin{table}[t]
\centering
\resizebox{.99\columnwidth}{!}{%
\begin{tabular}{lccc}
\toprule
FGSM & SVHN & CIFAR-10 & CIFAR-100 \\
\midrule
ERM & 23.41$\pm$0.569 & 36.06$\pm$1.908 & 68.78$\pm$0.699 \\
Dropout & 22.36$\pm$0.591 & 34.13$\pm$0.844 & 64.70$\pm$0.549 \\
Label Smoothing & 17.74$\pm$1.674 & \underline{23.24$\pm$0.427} & \underline{57.30$\pm$0.410} \\
Flooding & 17.40$\pm$0.656 & 36.42$\pm$1.303 & 68.45$\pm$0.407 \\
MixUp & 19.95$\pm$0.637 & 25.82$\pm$0.384 & 65.90$\pm$0.498 \\
Adv. Training & \textbf{14.33$\pm$0.200} & \textbf{18.58$\pm$0.304} & \textbf{48.51$\pm$0.260} \\
RMP & 23.73$\pm$0.965 & 35.40$\pm$0.572 & 68.52$\pm$0.515 \\
AMP & \underline{16.82$\pm$1.561} & 28.61$\pm$0.359 & 59.04$\pm$1.325 \\
\midrule
PGD & SVHN & CIFAR-10 & CIFAR-100 \\
\midrule
ERM & 45.17$\pm$1.085 & 58.88$\pm$2.296 & 85.46$\pm$0.770 \\
Dropout & 41.76$\pm$1.346 & 55.21$\pm$1.088 & 78.46$\pm$1.081 \\
Label Smoothing & 32.55$\pm$2.005 & \underline{34.93$\pm$0.443} & \underline{65.31$\pm$0.700} \\
Flooding & 33.50$\pm$1.707 & 60.32$\pm$1.393 & 84.66$\pm$0.285 \\
MixUp & 75.75$\pm$2.129 & 62.77$\pm$1.018 & 89.58$\pm$0.596 \\
Adv. Training & \textbf{20.20$\pm$0.409} & \textbf{21.46$\pm$0.373} & \textbf{51.72$\pm$0.327} \\
RMP & 44.74$\pm$0.960 & 58.06$\pm$0.650 & 84.80$\pm$0.488 \\
AMP & \underline{25.15$\pm$1.942} & 49.72$\pm$0.785 & 73.95$\pm$2.608 \\
\bottomrule
\end{tabular}
}%
\caption{Test errors (\%) against the while-box FGSM and PGD adversarial attacks. Each experiment has been run ten times to report the mean and standard derivation of errors.}
\label{tab:advresults}
\end{table}

The previous work \cite{zhao2020bridging} suggests that the flat minima make the adversarial attacks take more efforts for the input to leave the minima, so AMP is expected to improve the model's adversarial robustness. To validate this, we use the models trained with different regularization schemes to evaluate their adversarial robustness against the Fast Gradient Sign Method (FGSM) \cite{goodfellow2015explaining} and Projected Gradient Decent (PGD) \cite{madry2017towards} attacks. For FGSM, we set the perturbation radius to 4 per pixel. For PGD, we set the step size to 1 and perform 10 steps to generate adversarial examples, the perturbation radius is the same as FGSM. PreActResNet18 is chosen as the model architecture. We report the top-1 classification error on the adversarial examples constructed from the test set in Table~\ref{tab:advresults}. From the results, adversarial training outperforms all other schemes, since it directly trains models on the adversarial examples. AMP and label smoothing also show an effect in improving the model's robustness against both single-step FGSM attack and multi-step PGD attack.

\subsection{Loss Curve}

To investigate the mechanism of different regularization schemes in the training course, we plot the evolution curves of the training loss and the test loss in Figure~\ref{fig:loss} using PreActResNet18. We select ERM and two representative analogues (Flooding and MixUp) which achieved the second-best performance in the previous experiment to compare with AMP. From Figure~\ref{fig:loss}, ERM obtains the smallest training loss, and MixUp retains a high training loss since it trains models on augmented examples. AMP injects a small perturbation into the model parameter, and hence the training loss is slightly increased. It appears that the Flooding scheme affects training only when the training loss drops to a very low value, whereas MixUp and AMP take effects much earlier. For the test loss, AMP converges at a similar speed as other schemes, and reduces the test loss to a smaller value at the final stage.

\begin{figure*}[t]
\centering
\begin{subfigure}{0.66\columnwidth}%
\centering%
\captionsetup{width=0.9\columnwidth}%
\includegraphics[width=0.95\columnwidth]{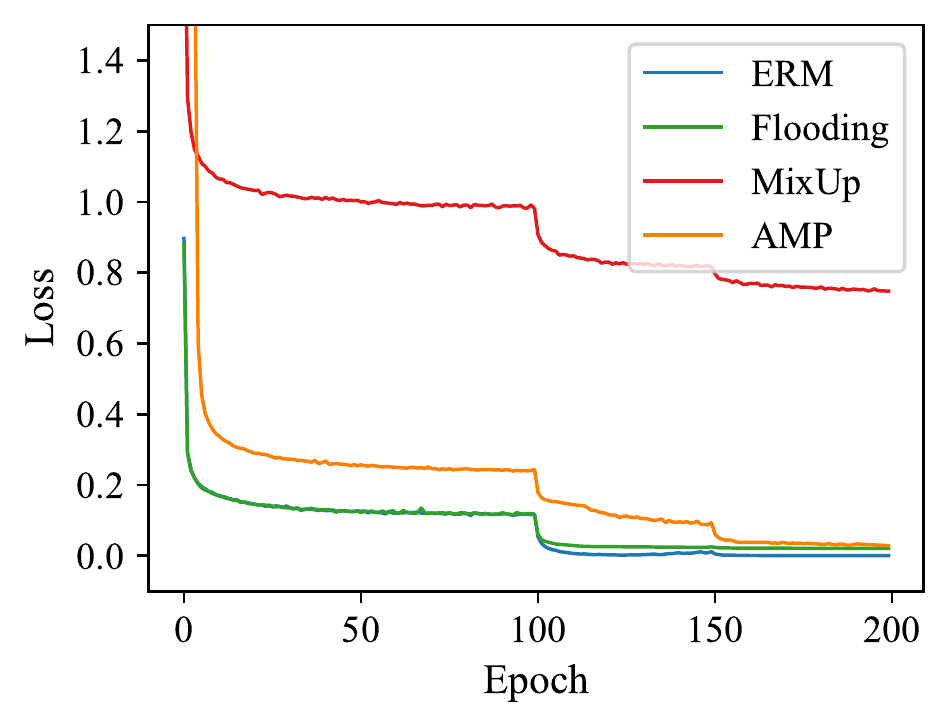}%
\caption{SVHN Training Loss}%
\end{subfigure}%
\begin{subfigure}{0.66\columnwidth}%
\centering%
\captionsetup{width=0.9\columnwidth}%
\includegraphics[width=0.95\columnwidth]{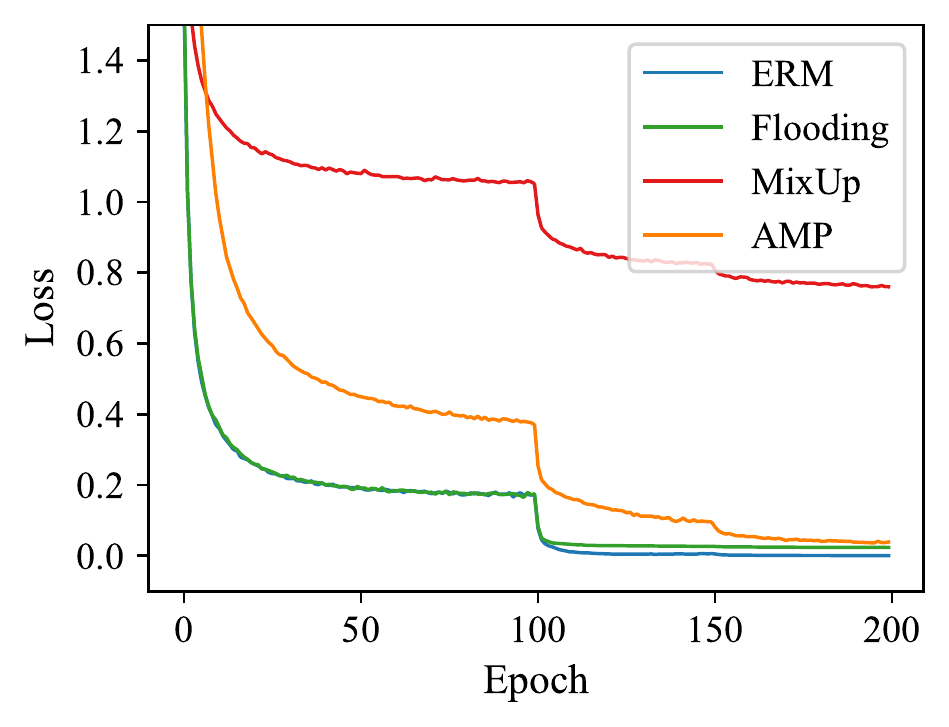}%
\caption{CIFAR-10 Training Loss}%
\end{subfigure}%
\begin{subfigure}{0.66\columnwidth}%
\centering%
\captionsetup{width=0.9\columnwidth}%
\includegraphics[width=0.95\columnwidth]{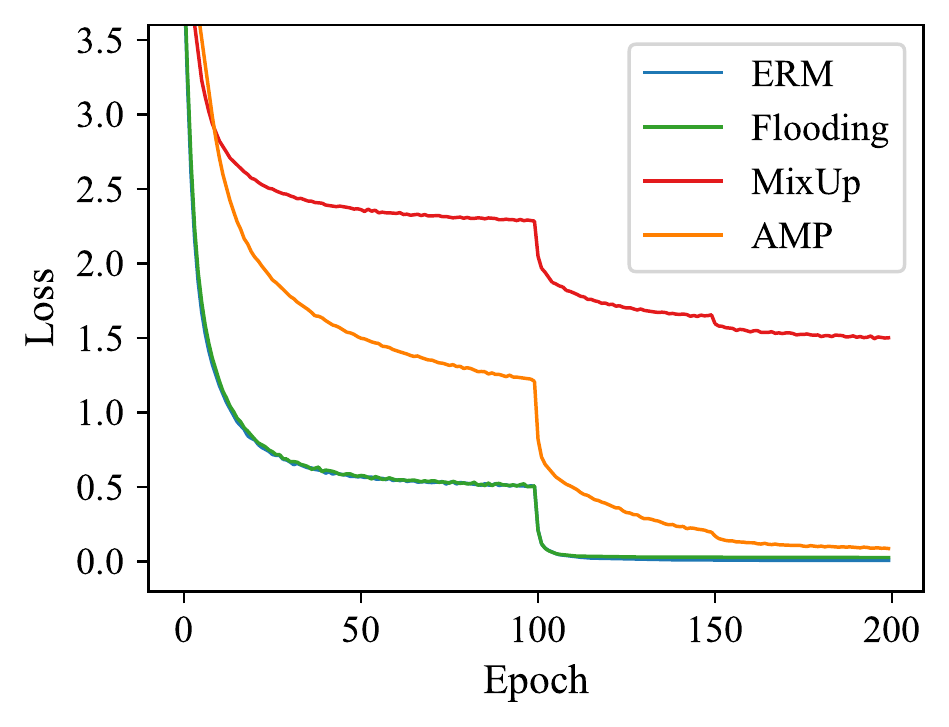}%
\caption{CIFAR-100 Training Loss}%
\end{subfigure}%
\\
\begin{subfigure}{0.66\columnwidth}%
\centering%
\captionsetup{width=0.9\columnwidth}%
\includegraphics[width=0.95\columnwidth]{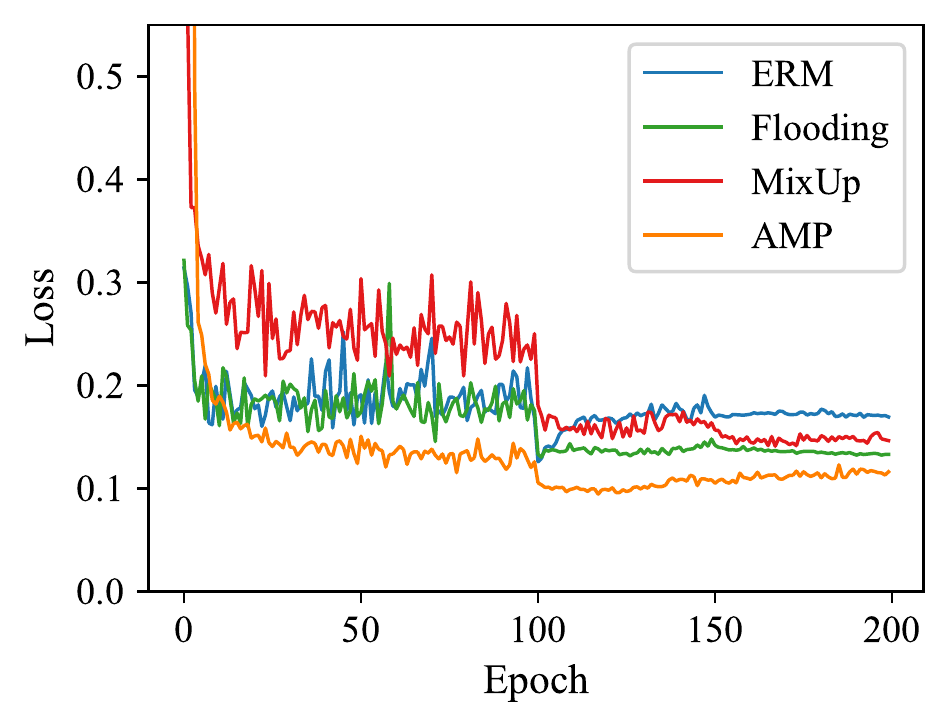}%
\caption{SVHN Test Loss}%
\end{subfigure}%
\begin{subfigure}{0.66\columnwidth}%
\centering%
\captionsetup{width=0.9\columnwidth}%
\includegraphics[width=0.95\columnwidth]{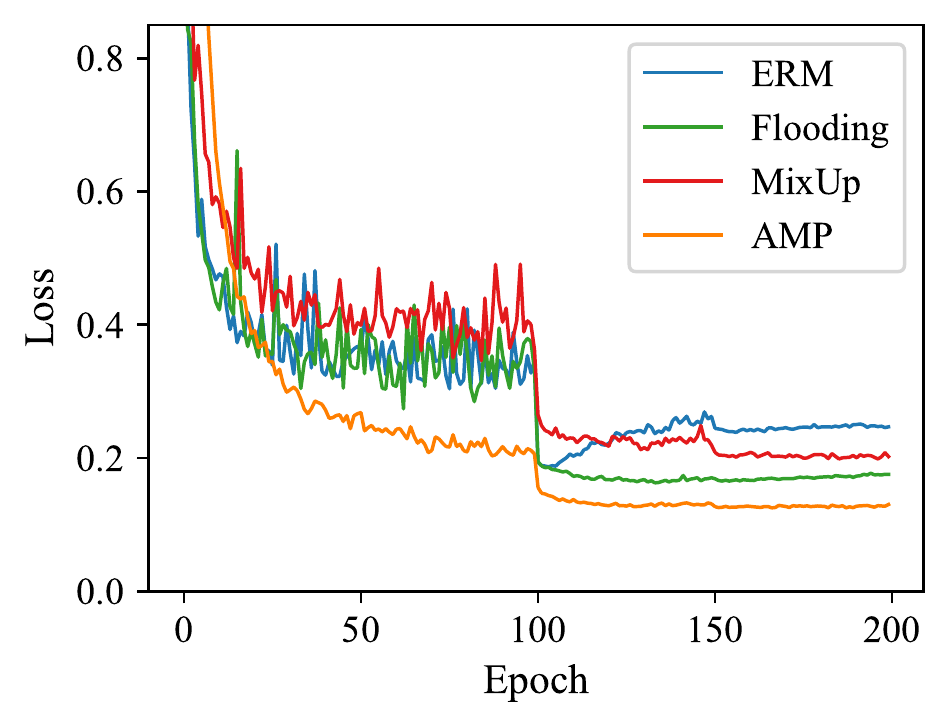}%
\caption{CIFAR-10 Test Loss}%
\end{subfigure}%
\begin{subfigure}{0.66\columnwidth}%
\centering%
\captionsetup{width=0.9\columnwidth}%
\includegraphics[width=0.95\columnwidth]{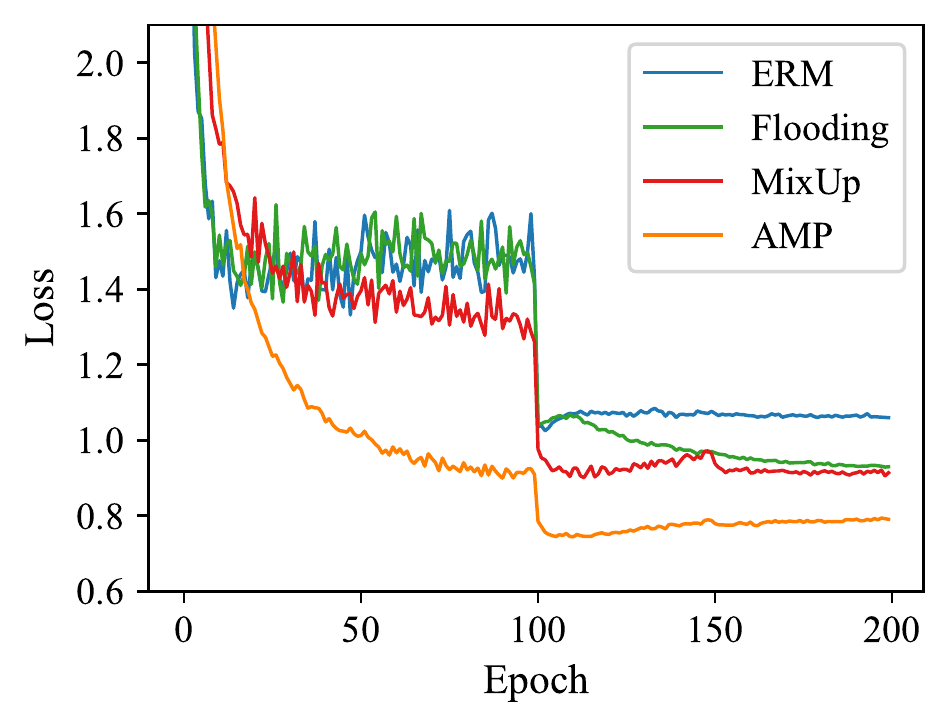}%
\caption{CIFAR-100 Test Loss}%
\end{subfigure}%
\caption{Loss curves for PreActResNet18 with different regularization schemes on three benchmark image datasets.}
\label{fig:loss}
\end{figure*}

\subsection{Flatness of Selected Minima}

We visualize the landscapes around the minima of the empirical risk selected by ERM or AMP, the 2D views are plotted in Figure~\ref{fig:flatness2d} and the 3D views are in Figure~\ref{fig:flatness3d}. Specifically, we compute the empirical risks of the PreActResNet18 models whose parameter is perturbed along two random directions $\boldsymbol{d}_x,\boldsymbol{d}_j$ with different step sizes $\delta_x,\delta_y$, where the direction vectors are normalized by the norm of filters suggested by \cite{li2018visualizing}. 
Specifically, we visualize the landscapes by computing
\begin{equation*}
\mathcal{L}_\mathrm{ERM}(\boldsymbol{\theta}^\ast+\delta_x\boldsymbol{d}_x+\delta_y\boldsymbol{d}_y)
\end{equation*}

The results suggest that AMP indeed selects flatter minima via adversarial perturbations.

\begin{figure*}[t]
\centering
\begin{subfigure}{0.66\columnwidth}%
\centering%
\captionsetup{width=0.9\columnwidth}%
\includegraphics[width=0.9\columnwidth]{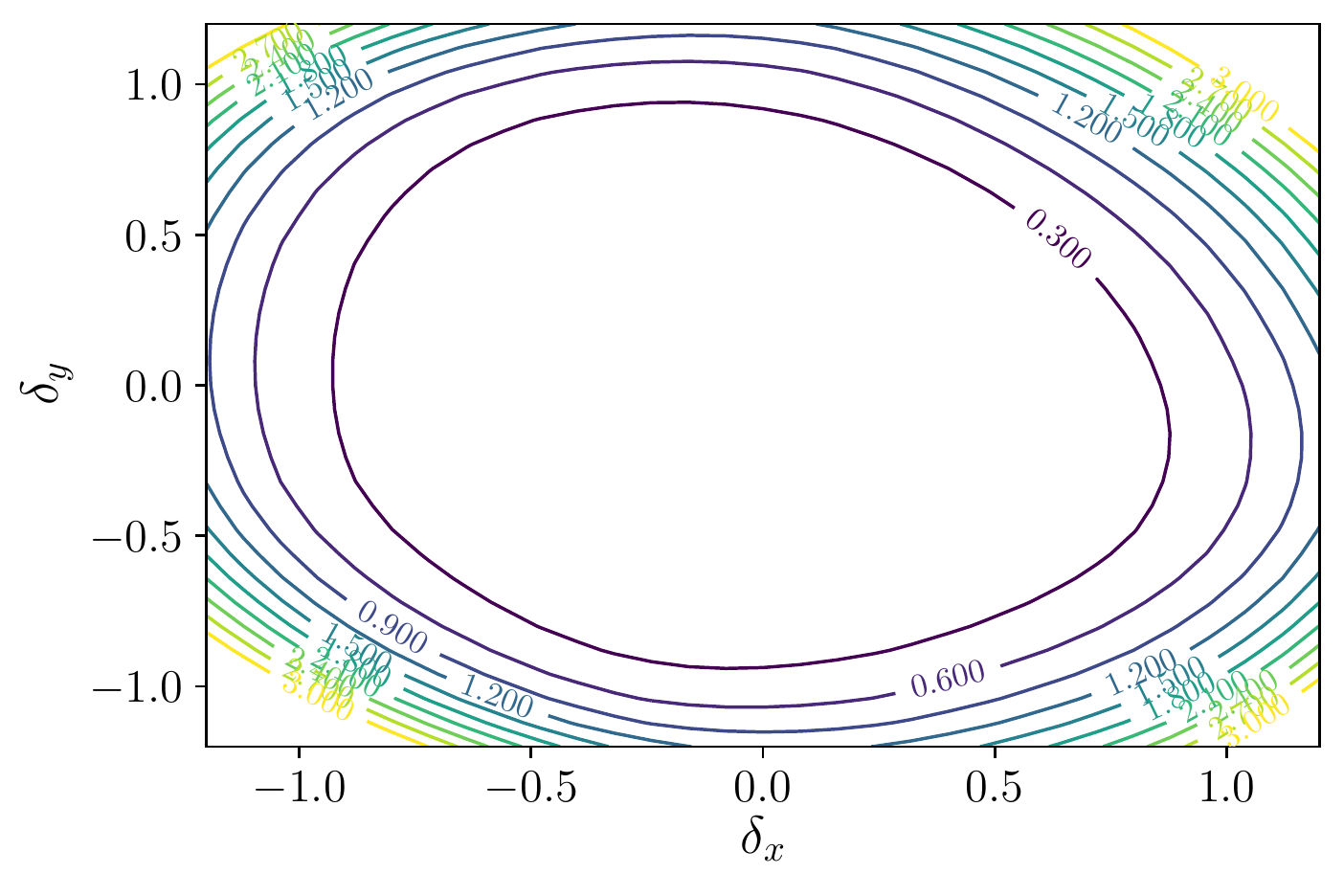}%
\caption{ERM Training Loss on SVHN}%
\end{subfigure}%
\begin{subfigure}{0.66\columnwidth}%
\centering%
\captionsetup{width=0.9\columnwidth}%
\includegraphics[width=0.9\columnwidth]{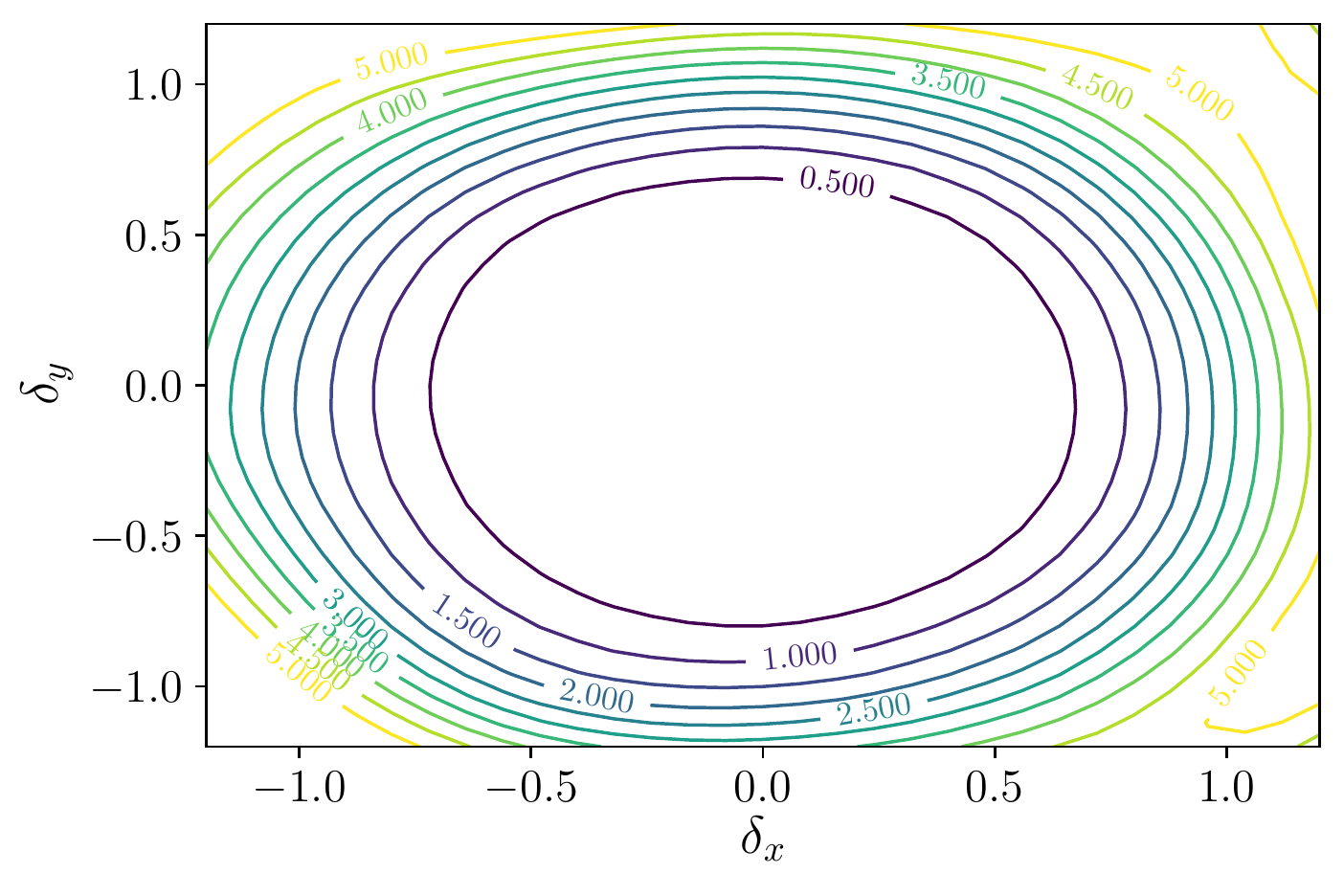}%
\caption{ERM Training Loss on CIFAR-10}%
\end{subfigure}%
\begin{subfigure}{0.66\columnwidth}%
\centering%
\captionsetup{width=0.9\columnwidth}%
\includegraphics[width=0.9\columnwidth]{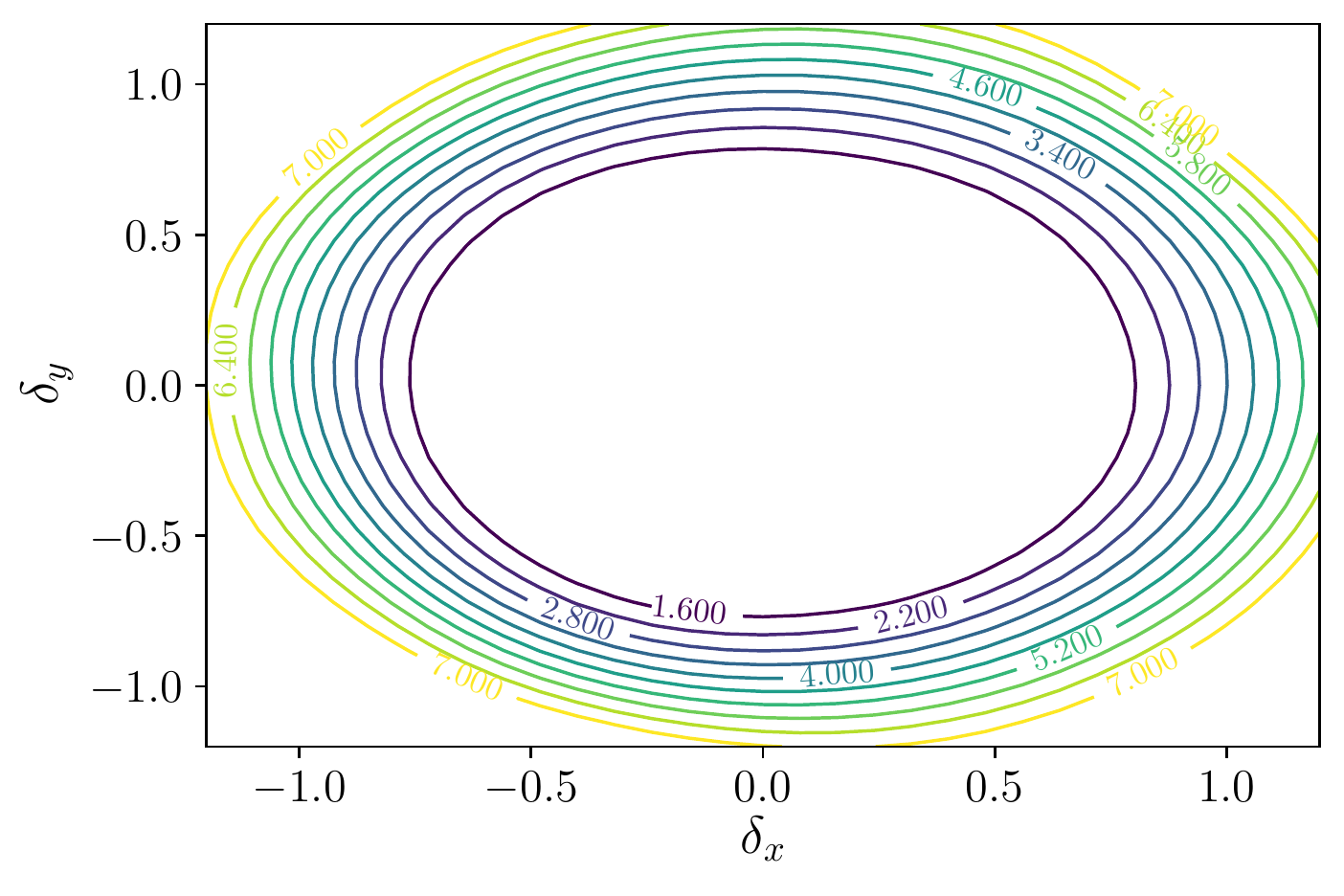}%
\caption{ERM Training Loss on CIFAR-100}%
\end{subfigure}%
\\
\begin{subfigure}{0.66\columnwidth}%
\centering%
\captionsetup{width=0.9\columnwidth}%
\includegraphics[width=0.9\columnwidth]{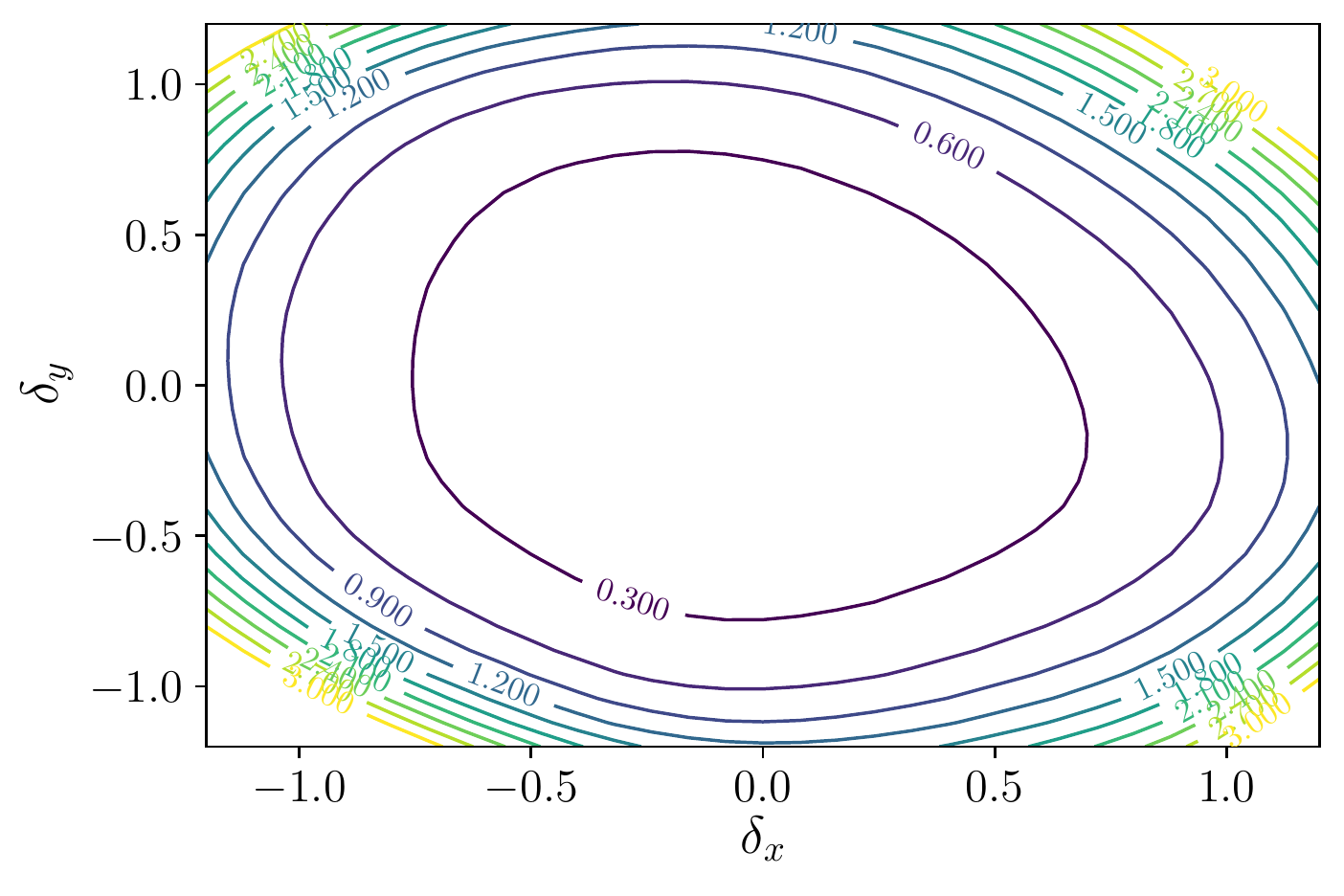}%
\caption{ERM Test Loss on SVHN}%
\end{subfigure}%
\begin{subfigure}{0.66\columnwidth}%
\centering%
\captionsetup{width=0.9\columnwidth}%
\includegraphics[width=0.9\columnwidth]{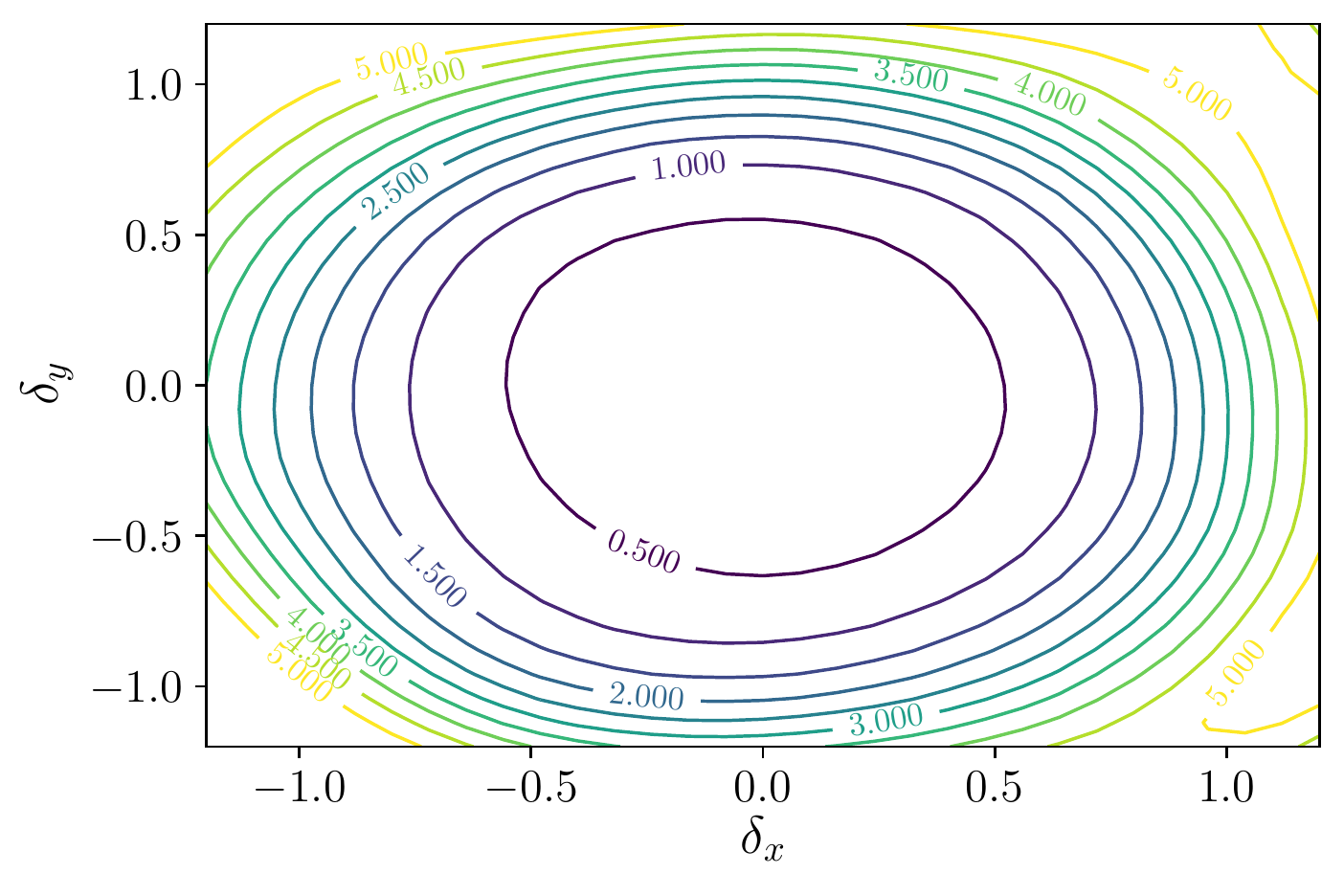}%
\caption{ERM Test Loss on CIFAR-10}%
\end{subfigure}%
\begin{subfigure}{0.66\columnwidth}%
\centering%
\captionsetup{width=0.9\columnwidth}%
\includegraphics[width=0.9\columnwidth]{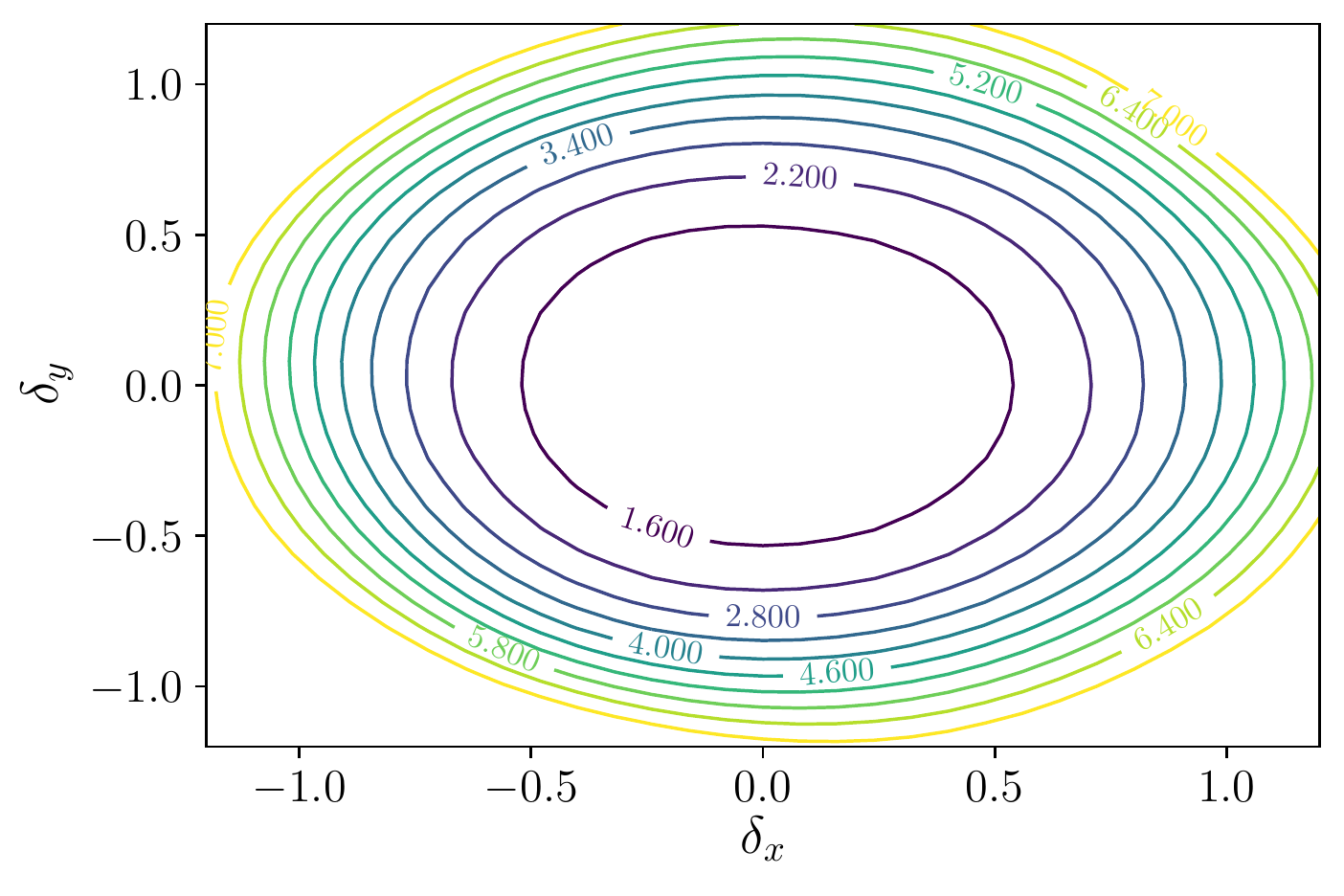}%
\caption{ERM Test Loss on CIFAR-100}%
\end{subfigure}%
\\
\begin{subfigure}{0.66\columnwidth}%
\centering%
\captionsetup{width=0.9\columnwidth}%
\includegraphics[width=0.9\columnwidth]{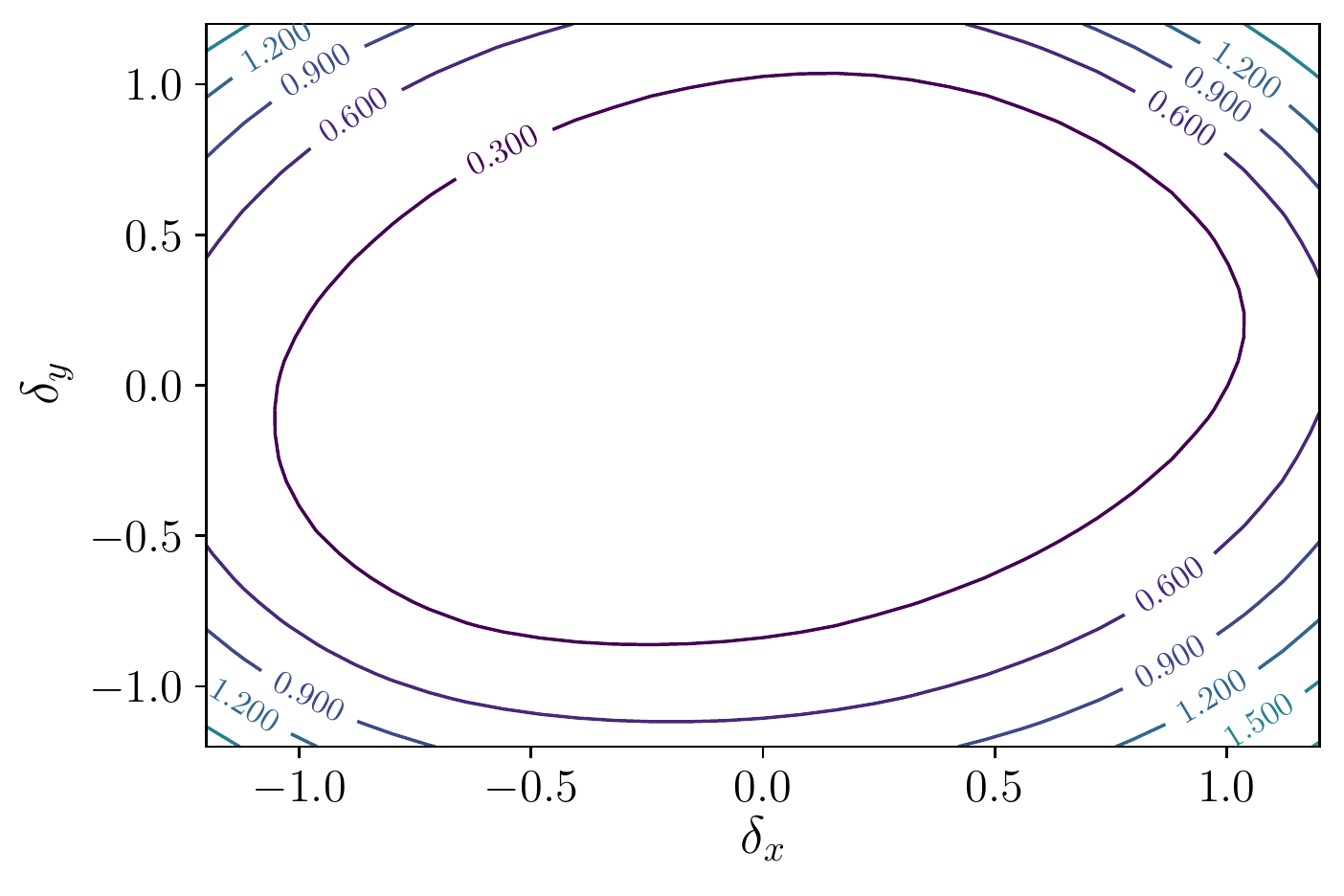}%
\caption{AMP Training Loss on SVHN}%
\end{subfigure}%
\begin{subfigure}{0.66\columnwidth}%
\centering%
\captionsetup{width=0.9\columnwidth}%
\includegraphics[width=0.9\columnwidth]{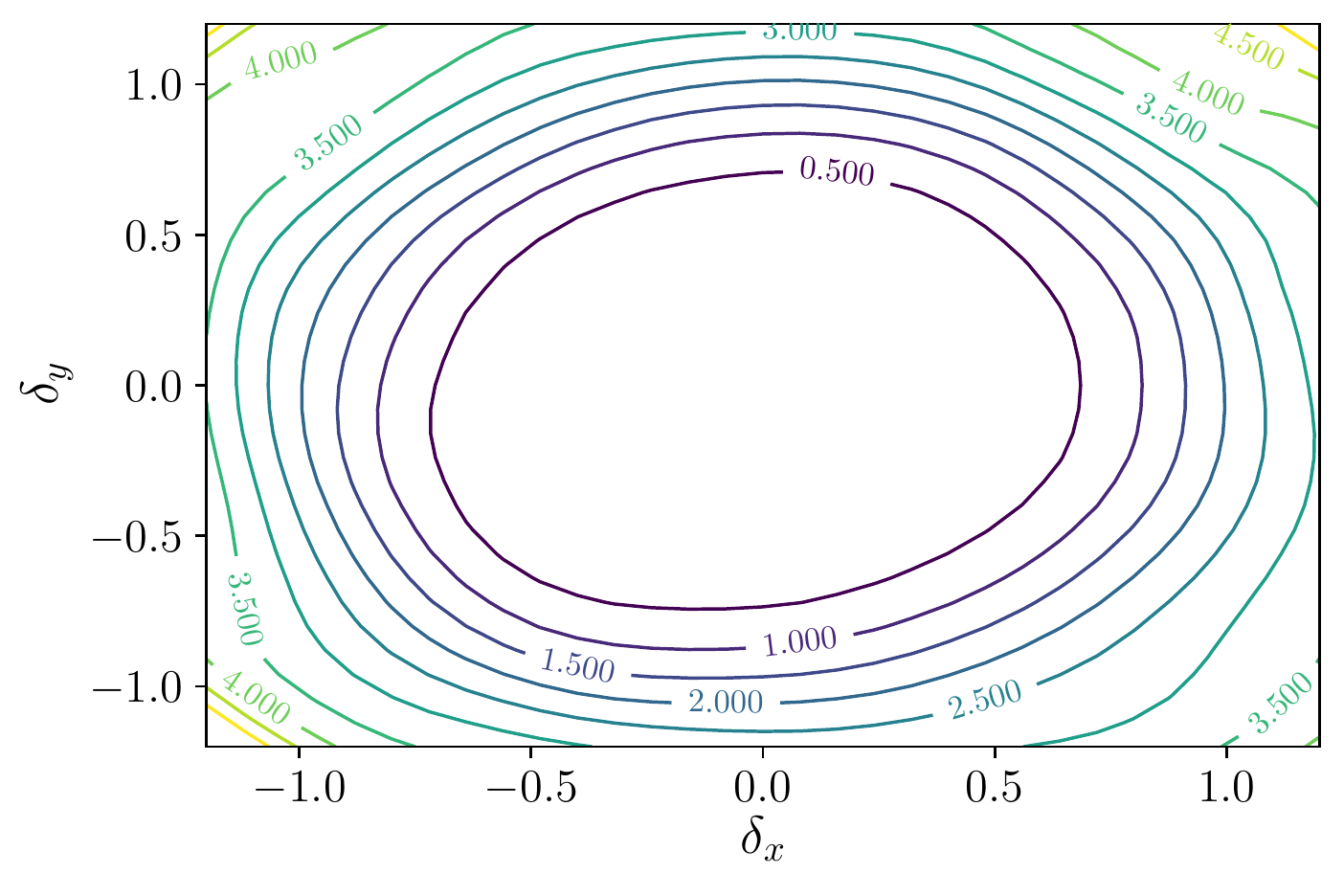}%
\caption{AMP Training Loss on CIFAR-10}%
\end{subfigure}%
\begin{subfigure}{0.66\columnwidth}%
\centering%
\captionsetup{width=0.9\columnwidth}%
\includegraphics[width=0.9\columnwidth]{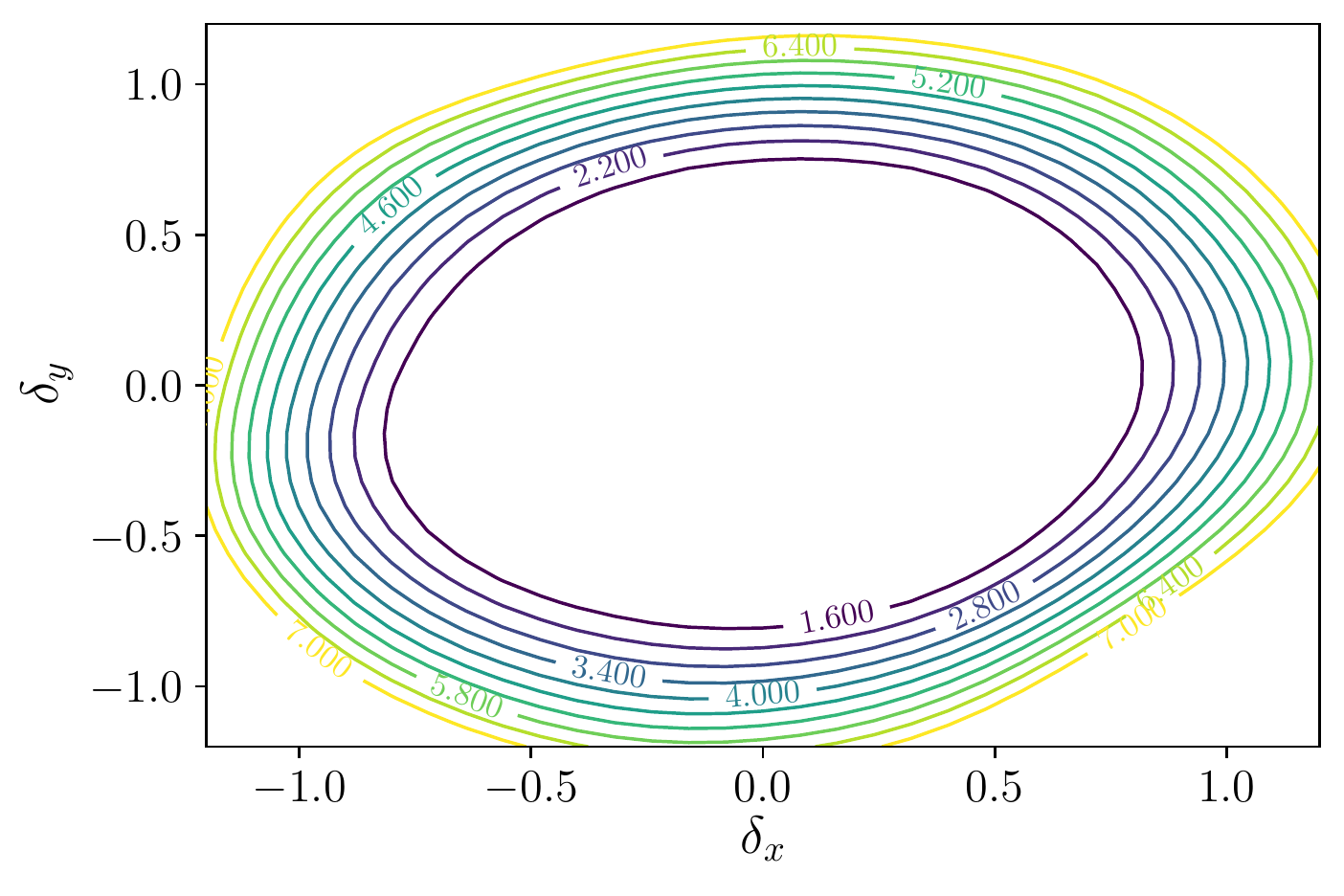}%
\caption{AMP Training Loss on CIFAR-100}%
\end{subfigure}%
\\
\begin{subfigure}{0.66\columnwidth}%
\centering%
\captionsetup{width=0.9\columnwidth}%
\includegraphics[width=0.9\columnwidth]{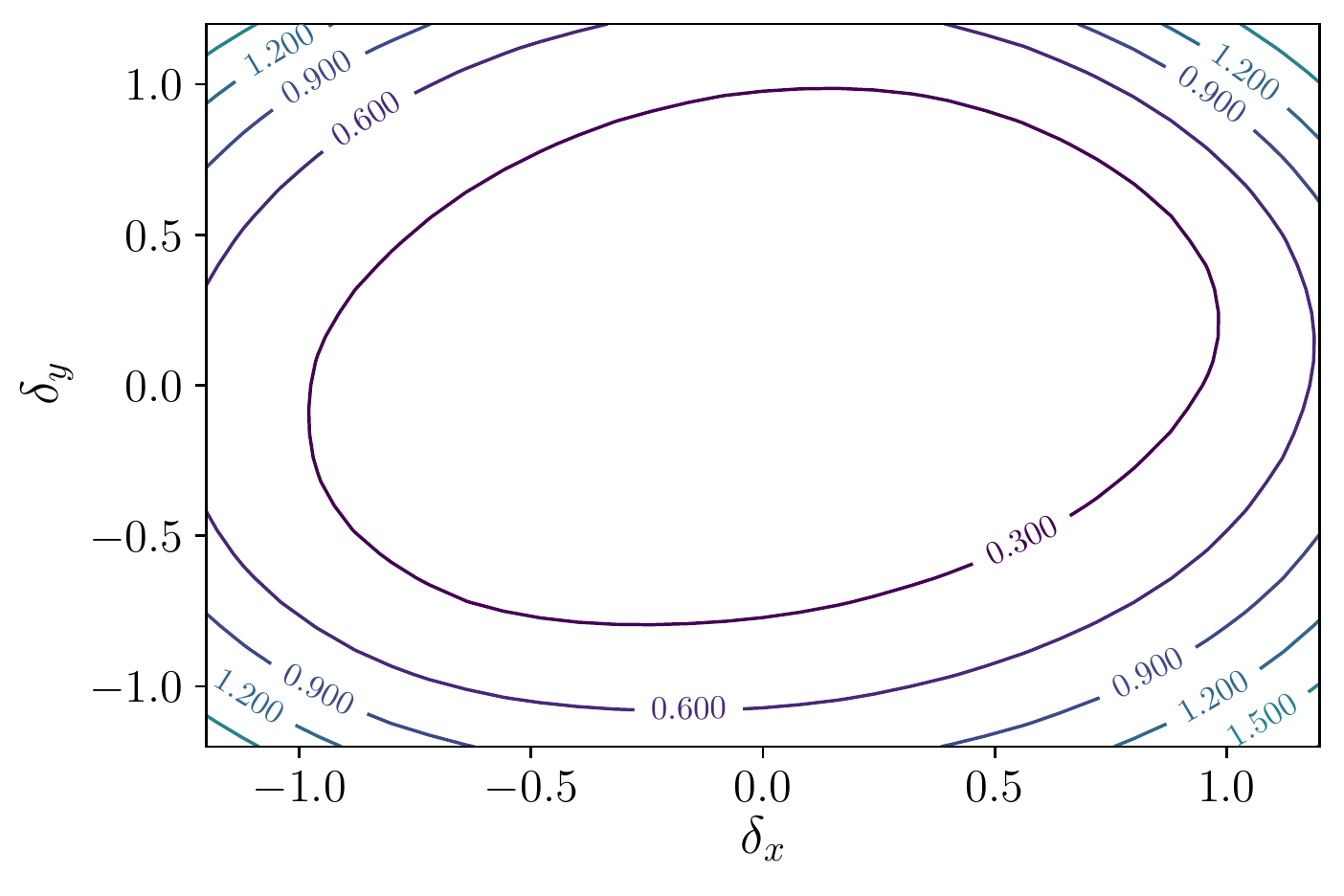}%
\caption{AMP Test Loss on SVHN}%
\end{subfigure}%
\begin{subfigure}{0.66\columnwidth}%
\centering%
\captionsetup{width=0.9\columnwidth}%
\includegraphics[width=0.9\columnwidth]{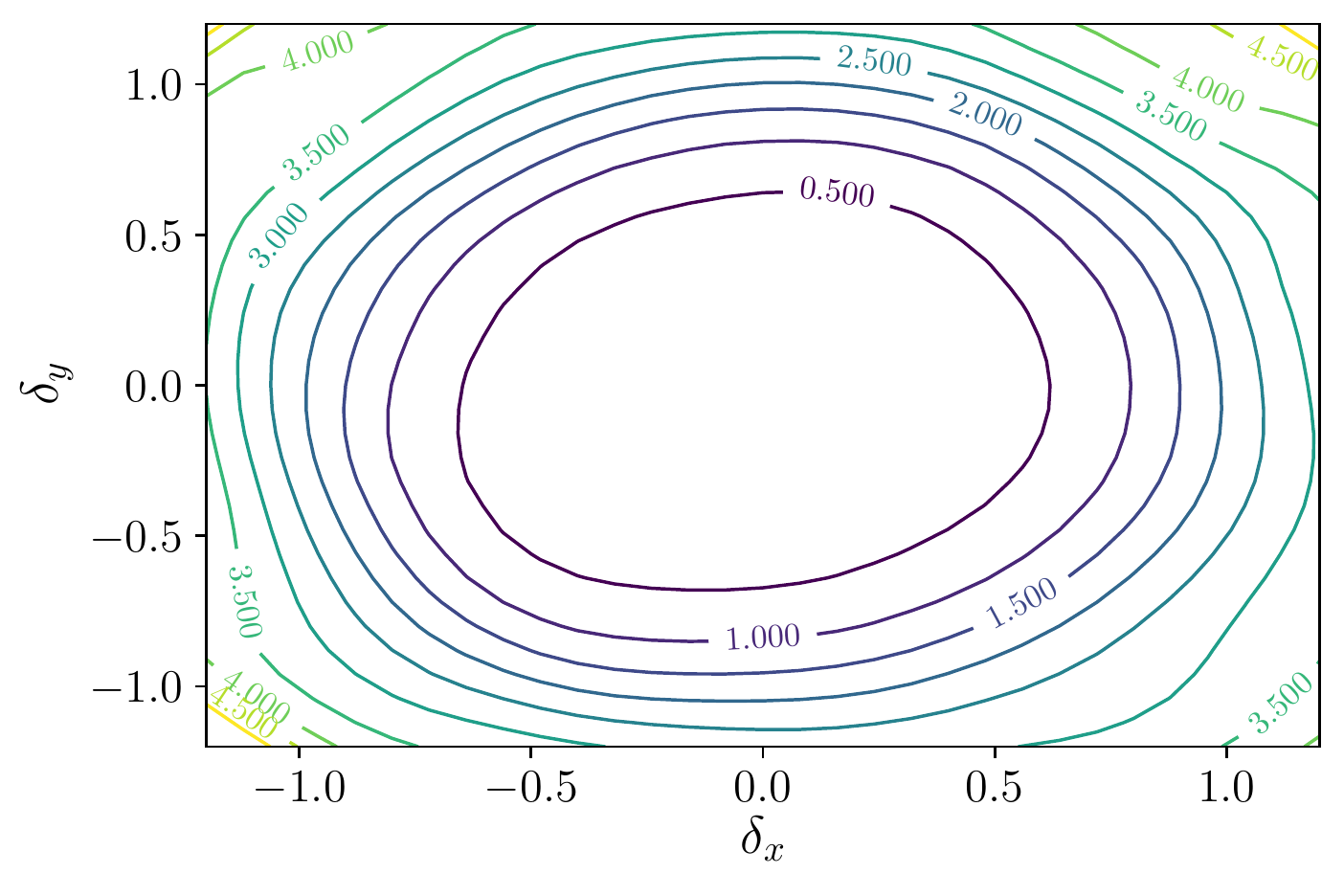}%
\caption{AMP Test Loss on CIFAR-10}%
\end{subfigure}%
\begin{subfigure}{0.66\columnwidth}%
\centering%
\captionsetup{width=0.9\columnwidth}%
\includegraphics[width=0.9\columnwidth]{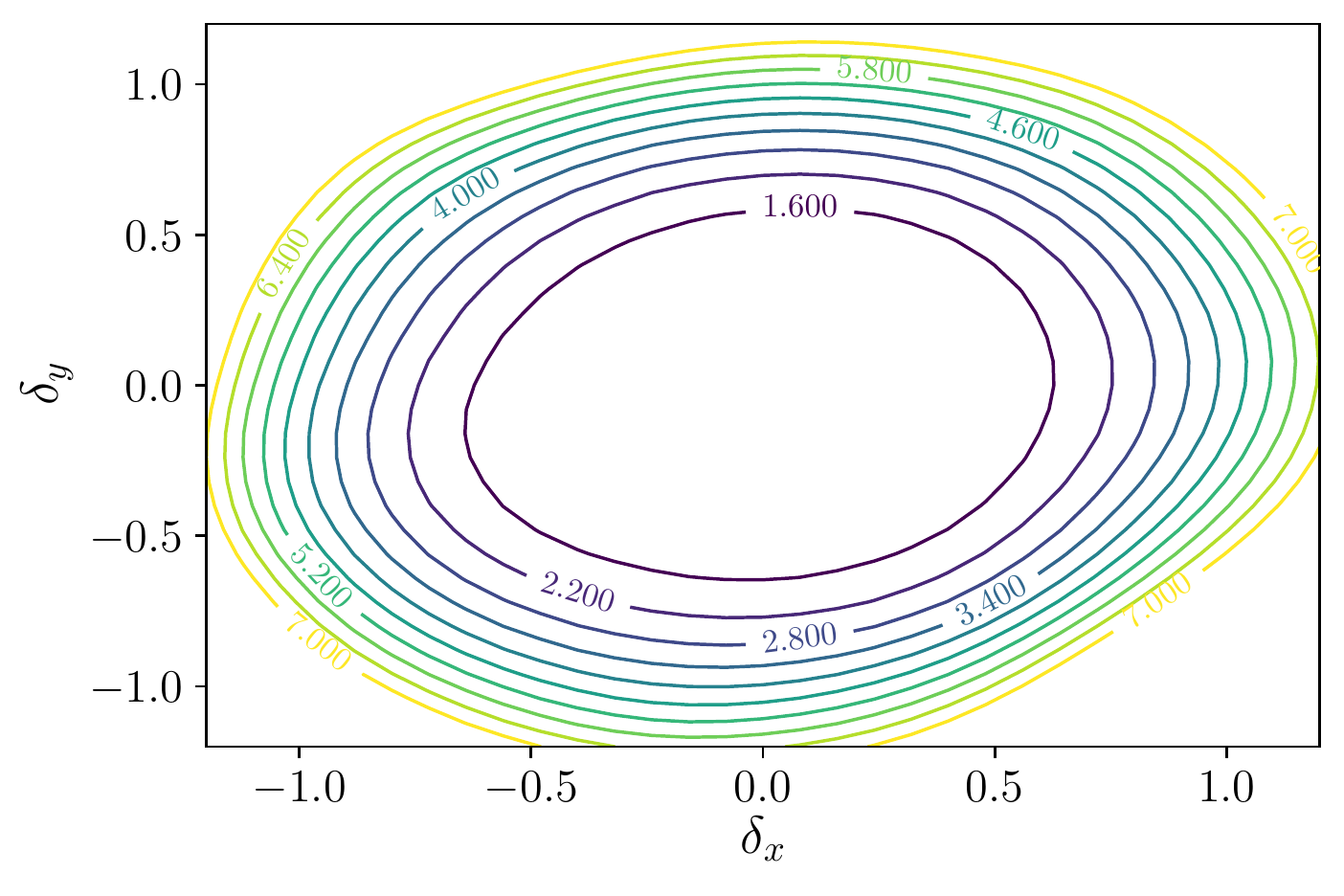}%
\caption{AMP Test Loss on CIFAR-100}%
\end{subfigure}%
\caption{2D visualization of the minima of the empirical risk selected by ERM and AMP on three benchmark image datasets.}
\label{fig:flatness2d}
\end{figure*}

\begin{figure*}[t]
\centering
\begin{subfigure}{0.9\columnwidth}%
\centering%
\captionsetup{width=0.9\columnwidth}%
\includegraphics[width=0.9\columnwidth]{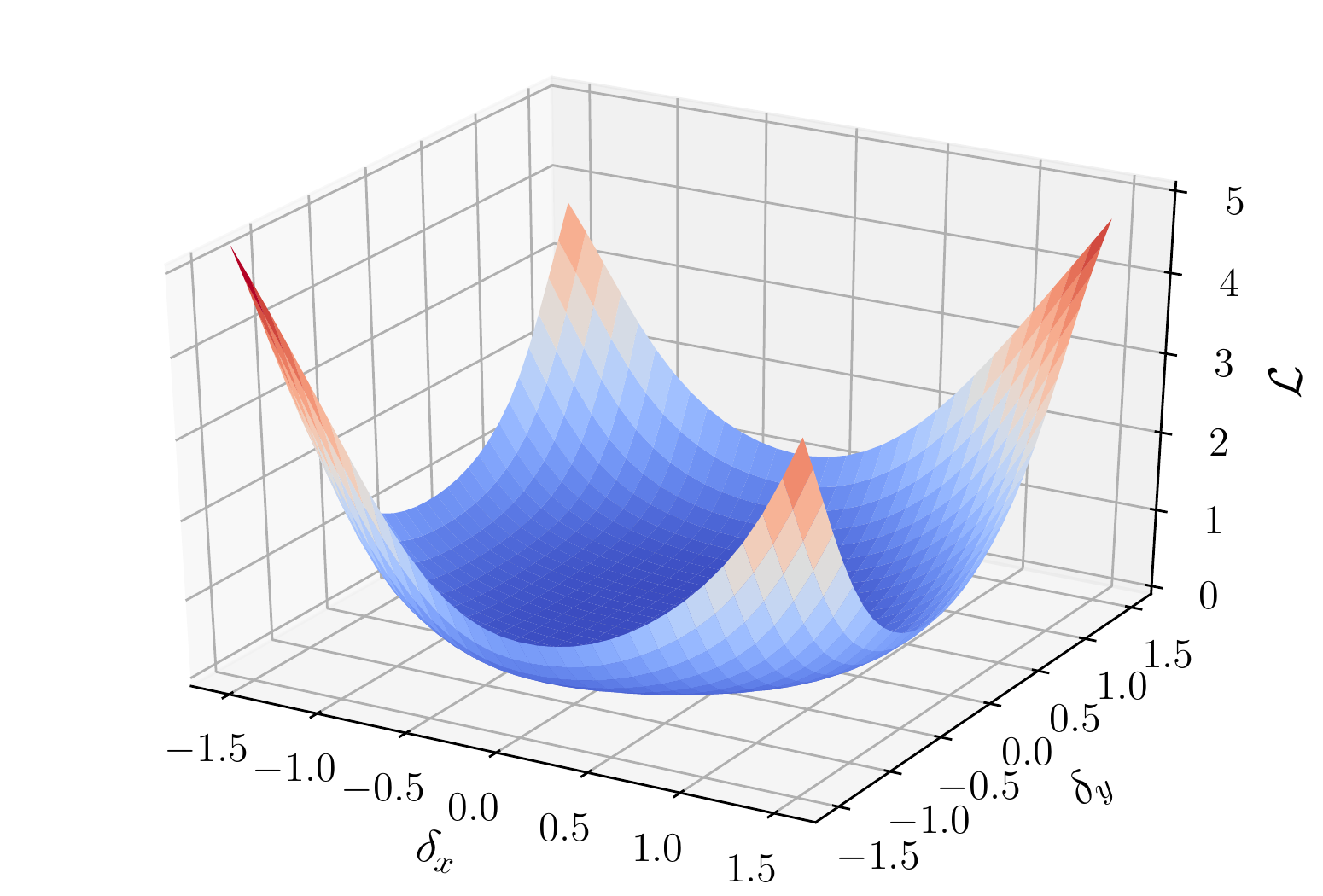}%
\caption{ERM Training Loss}%
\end{subfigure}%
\begin{subfigure}{0.9\columnwidth}%
\centering%
\captionsetup{width=0.9\columnwidth}%
\includegraphics[width=0.9\columnwidth]{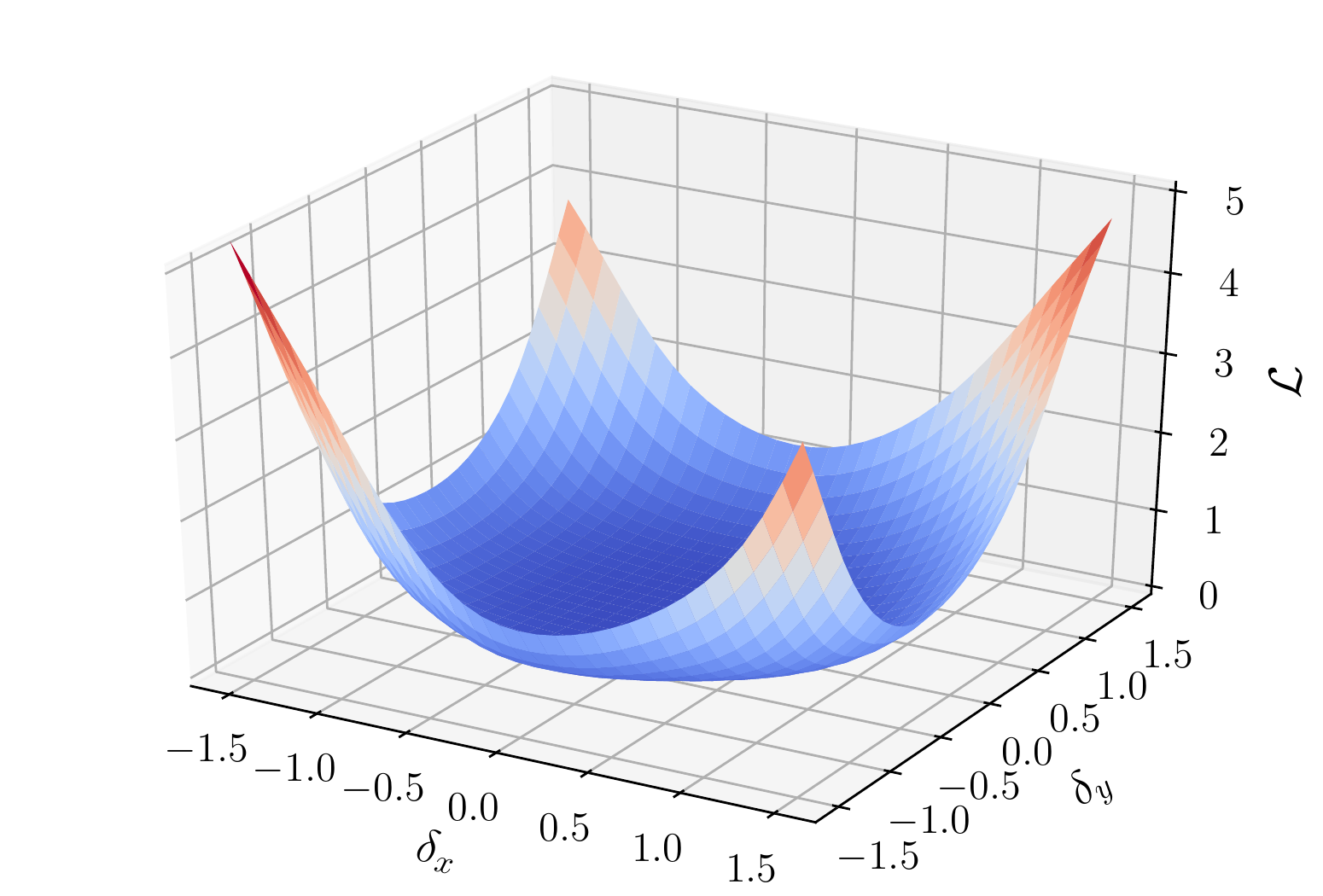}%
\caption{ERM Test Loss}%
\end{subfigure}%
\\
\begin{subfigure}{0.9\columnwidth}%
\centering%
\captionsetup{width=0.9\columnwidth}%
\includegraphics[width=0.9\columnwidth]{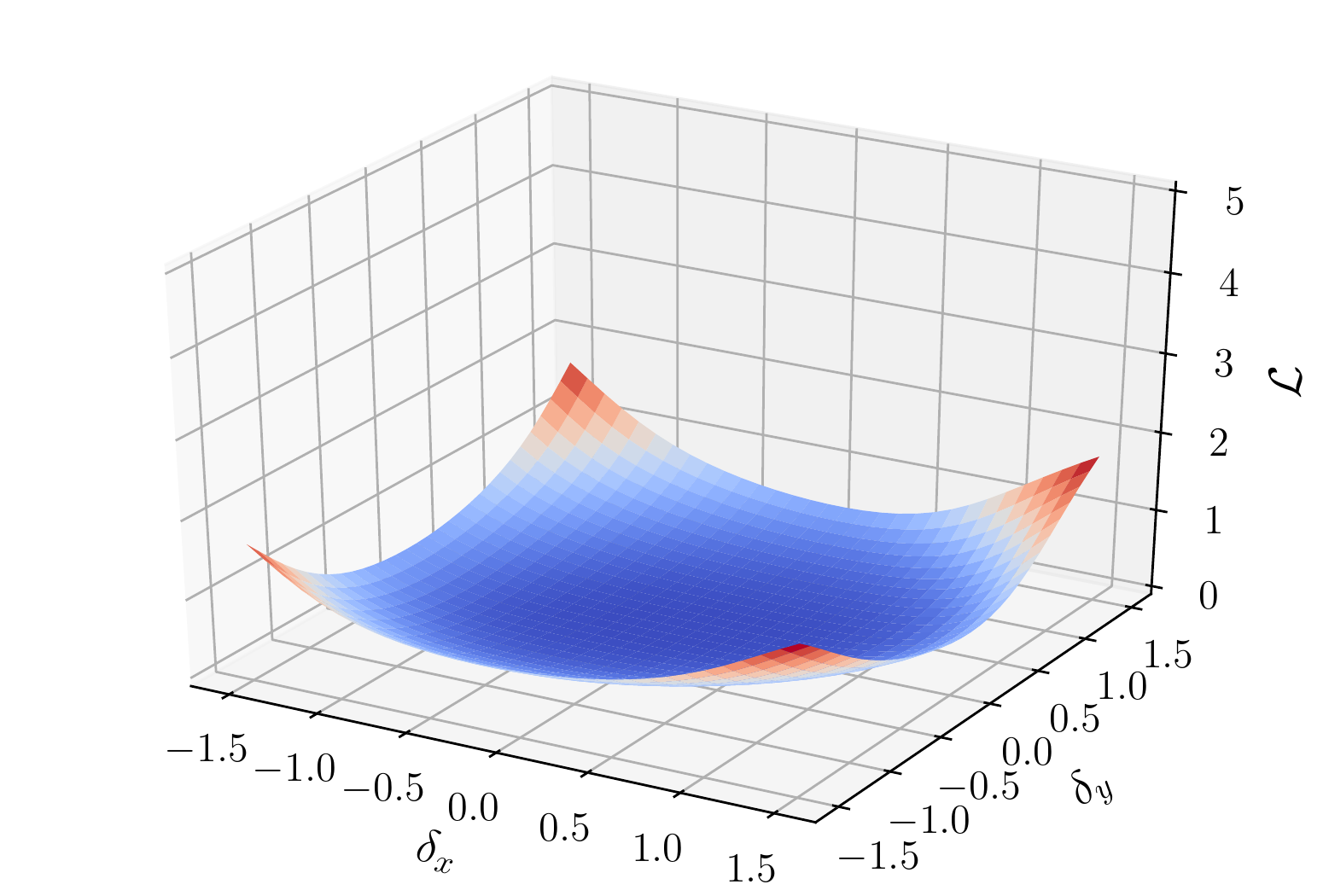}%
\caption{AMP Training Loss}%
\end{subfigure}%
\begin{subfigure}{0.9\columnwidth}%
\centering%
\captionsetup{width=0.9\columnwidth}%
\includegraphics[width=0.9\columnwidth]{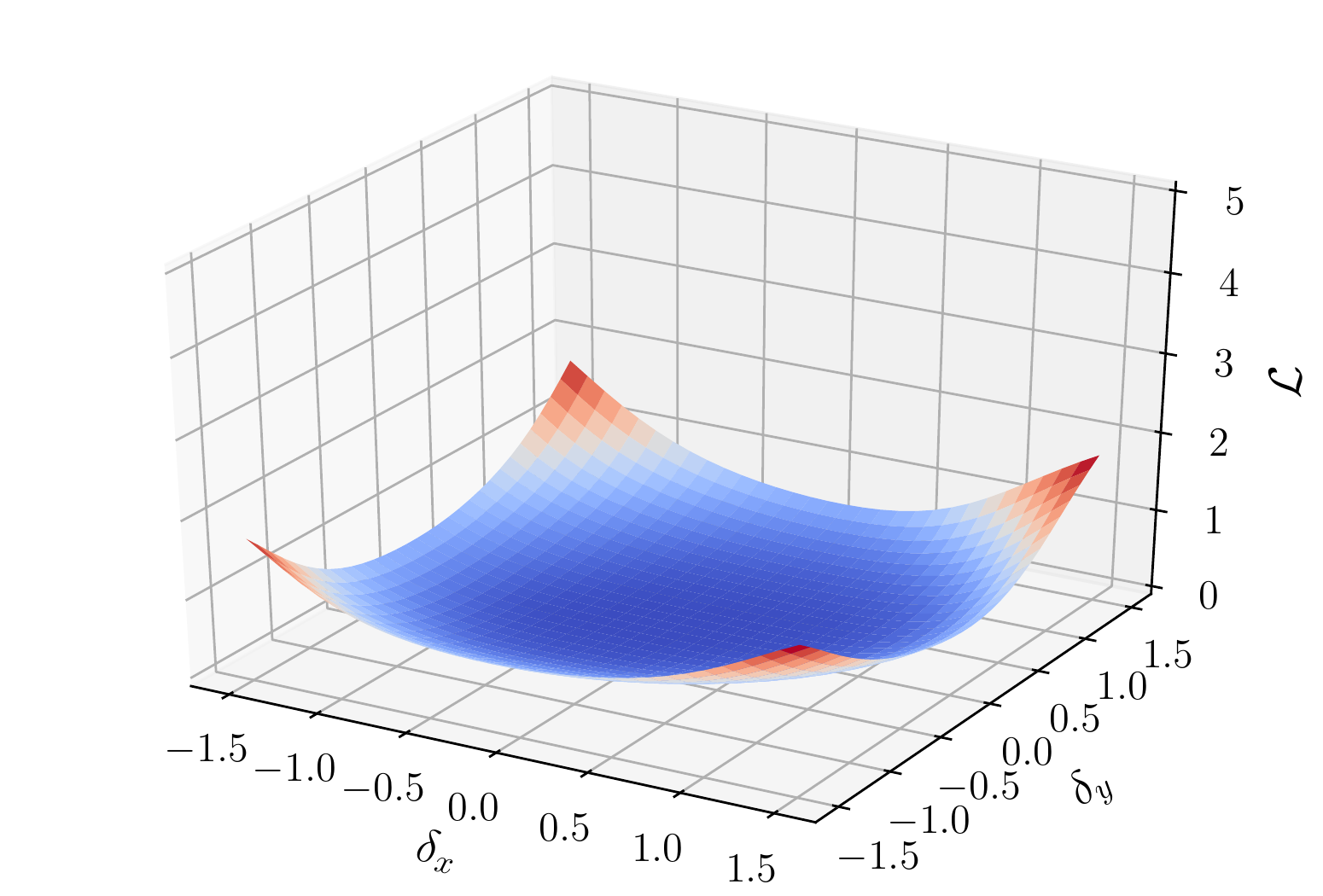}%
\caption{AMP Test Loss}%
\end{subfigure}%
\caption{3D visualization of the minima of the empirical risk selected by ERM and AMP on the SVHN dataset.}
\label{fig:flatness3d}
\end{figure*}

\subsection{Computing Environment and Resources}

Our PyTorch code is executed in a CUDA environment. When evaluated on a single Tesla V100 GPU, the code takes around 2.4 hours to train a PreActResNet18 model with ERM on the CIFAR-10 dataset, and around 4.2 hours with AMP. The computation time mainly depends on the number of inner iterations, the number of epochs, and the number of GPUs. The code and datasets for reproduction can be found at \url{https://github.com/hiyouga/AMP-Regularizer}.

\end{document}